\documentclass{article}

\usepackage[preprint]{neurips_2026}
\usepackage{natbib}
\setcitestyle{numbers,square,comma}

\usepackage[utf8]{inputenc} 
\usepackage[T1]{fontenc}    
\usepackage{hyperref}       
\usepackage{url}            
\usepackage{dirtree}
\usepackage{booktabs}       
\usepackage{amsfonts}       
\usepackage{nicefrac}       
\usepackage{microtype}      
\usepackage{xcolor}         
\usepackage{graphicx}
\usepackage{subcaption}

\usepackage{algorithm}
\usepackage{algpseudocodex}

\usepackage{macros_IJ}

\usepackage{amsmath}
\usepackage{amssymb}
\usepackage{mathtools}
\usepackage{amsthm}
\usepackage{geometry}

\usepackage[capitalize,noabbrev]{cleveref}

\theoremstyle{plain}
\newtheorem{theorem}{Theorem}[section]

\newtheorem{lemma}[theorem]{Lemma}
\newtheorem{corollary}[theorem]{Corollary}
\theoremstyle{definition}
\newtheorem{definition}[theorem]{Definition}

\theoremstyle{remark}
\newtheorem{remark}[theorem]{Remark}

\usepackage[textsize=tiny]{todonotes}

\title{Cost-Ordered Feasibility for Multi-Armed Bandits with Cost Subsidy}

\author{
Ishank Juneja \quad Carlee Joe-Wong \quad Osman Ya\u{g}an\\
Department of Electrical and Computer Engineering\\
Carnegie Mellon University\\
Pittsburgh, PA 15213\\
\texttt{\{ijuneja,cjoewong,oyagan\}@andrew.cmu.edu}
}

\begin{document}

\maketitle

\begin{abstract}
The classic multi-armed bandit (MAB) problem tackles the challenge of accruing maximum reward while making decisions under uncertainty. However, in applications, often the goal is to minimize cost subject to a constraint on the minimum permissible reward, an objective captured by multi-armed bandits with cost-subsidy (MAB-CS). Of interest to this paper is the setting where the quality (reward) constraint is specified relative to the unknown best reward and the cost of each arm is known. We characterize the expected sub-optimal samples required by any policy by proving instance-dependent lower bounds that offer new insight into the problem and are a strict generalization of prior bounds. Then, we propose an algorithm called Cost-Ordered Feasibility (COF) that leverages our insight and intelligently combine samples from all arms to gauge the feasibility of a cheap arm. Thereafter, we analyze COF to establish instance-dependent upper bounds on its expected cumulative cost and quality regret, i.e., relative to the cheapest feasible arm. Finally, we empirically validate the merits of COF, comparing it to baselines from the literature through extensive simulation experiments on the MovieLens and Goodreads datasets as well as representative synthetic instances. Not only does our paper develop qualitatively better theoretical regret upper bounds, but COF also convincingly demonstrates improved empirical performance.
\end{abstract}

\section{Introduction}
\label{sec:intro}

Multi-armed bandits (MABs) \citep{lattimore2020bandit} is a framework for online sequential decision making that models decisions as arms that can be pulled. Pulling or sampling each arm leads to the observation of a scalar reward, the generative process for which is apriori unknown. The stationary stochastic bandits \citep{auer2002finite} framework in addition imposes the structure that the observed rewards be random variables that are independent and identically distributed. Classically in MABs, the goal is to maximize the cumulative reward across all the sampling decisions made over the problem horizon $T$. However, in practice, the costs associated with every sampling decision must be considered as part of the decision making. In particular, the goal in cost sensitive applications, like machine learning inference, is to minimize the cumulative cost subject to a constraint on the observed reward or quality level. For this reason, model routers are needed to map queries to the most appropriate LLM from a large selection of open source, proprietary, and task-specialized models \citep{laufer2025anatomymachinelearningecosystem}. The cost of these LLMs may vary over orders of magnitude, therefore routing must balance between expected quality and cost \citep{jitkrittum2025universal, tsiourvas2025causal, wei2025learning, wuefficient}.

\textbf{Cost-subsidy framework:} The quality-constrained cost-minimization setting is captured by the recently introduced multi-armed bandits with cost-subsidy (MAB-CS) framework \citep{pmlr-v130-sinha21a,juneja2025pairwise}. MAB-CS captures the quality constraint through threshold $\mu_{\CS}$. For a $K$-armed bandit instance with arms $\mathcal{A} = \{  a_1, a_2, \ldots, a_K \}$, the expected reward from sampling arm $k$ is denoted $\mu_k$. The constraint on the expected reward gives rise to a set of feasible arms $\mathcal{S} = \{  a_k \in \mathcal{A} \, | \, \mu_k \geq \mu_{\CS} \}$. In our work $\mu_{\CS}$ is specified in relation to the unknown best reward as $\mu_{\CS} = (1 - \alpha) \mu^*$, where $\alpha \in (0,1)$ is the {\em subsidy factor}, and $\mu^*$ is the best reward defined as $\max_{ a_k \in \mathcal{A} } \mu_k$. As is convention in bandits, we use $i^*$ to denote the arm bearing $\mu^*$. If the cost associated with sampling arm $a_k$ is denoted $c_k$, then in MAB-CS, the optimal arm $a^*$ is defined to be the cheapest arm from the set of feasible arms. Mathematically $a^* = \argmin_{a_i \in \mathcal{S}} c_i$ and reward borne by $a^*$ is denoted $\mu_{a^*}$. Application to the problem of selecting communication modalities in advertising \citep{javan2018hybrid} demonstrates the utility of MAB-CS. Available modalities like text message, postal mail, and doorstep solicitation shall each have an unknown expected quality (i.e., effectiveness or response rate) and a known deployment cost \citep{moffett2021theory}. The business may want to program their response rate to be a fraction of the best unknown rate across modalities. This relative rate then serves as the quality constraint subject to which cost must be minimized. The MAB-CS framework is immensely useful for applications where no preset threshold has to be satisfied but the run-time performance of the system cannot be permitted to fall too far below the best available alternatives, no matter their cost. In the cost-subsidy framework decisions are made to minimize the metrics of cumulative cost and quality regret which are formally defined in Section~\ref{sec:theory_lower_bounds}. The former captures the excess cost paid over $c_{a^*}$ and the latter captures violations of the quality constraint. The core challenge in jointly minimizing both regret metrics is exploring for cheap feasible arms while simultaneously estimating the feasibility threshold $(1 - \alpha) \mu^*$. 

\textbf{Gaps and challenges in prior work:} Prior work by \citet{pmlr-v130-sinha21a} identified {\em worst-case} lower bounds on cost and quality regret for the MAB-CS problem. While these lower bounds formally establish that cost and quality regret minimization in the framework is a strictly more challenging problem than regret minimization in conventional MABs, they do not provide insight into the design of practically useful algorithms. Other research has introduced lower bounds on cost and quality regret that inform algorithm design by revealing the dependence between the problem instance and the minimum samples required under expectation \citep{juneja2025pairwise}. However, their singular treatment of the best reward arm $i^*$ leaves open room for improved general analysis. Moreover, \citet{juneja2025pairwise} introduced regret-minimization policies for settings in which $\mu_{\CS}$ is either a known threshold, the subsidized reward of a specified reference arm, or the subsidized best reward. While their algorithms for the first two settings were order-wise optimal, their approach for the subsidized best reward setting studied in our work was not. A primary obstacle in regret minimization for this setting is determining if a candidate arm has reward that is feasible relative to $(1 - \alpha) \mu^*$. Prior work \citep{juneja2025pairwise} addressed this challenge through a two stage approach where first $i^*$ is identified through a best arm identification (BAI) scheme. Then, the feasibility of candidate arms is determined through pairwise comparisons. In this paper to estimate the $(1 - \alpha) \mu^*$ criterion we develop a method that progresses in a manner inspired by BAI without converging onto a single arm, avoiding the accrual of unnecessary samples.

\textbf{Our contributions} and their organization is as follows. In Section~\ref{sec:theory_lower_bounds}, we formalize the MAB-CS framework and provide instance-dependent lower bounds on the expected number of samples of all sub-optimal arms. We include an {\em entirely novel result} that jointly lower bounds the expected samples from a collection of arms. The joint bound is a structure that emerges naturally from the fact that an arm can be certified as infeasible once its reward is worse than the subsidized reward of {\em any} arbitrary arm, not just $i^*$. In Section~\ref{sec:cof} we present the {\em Cost-Ordered Feasibility (COF) algorithm}, which aggregates the signal from the comparisons between a candidate feasible arm and all other arms. Section~\ref{sec:cof} also includes an {\em upper bound on the expected cumulative cost and quality regrets} from COF. We validate the superiority of our approach on both regret metrics through {\em extensive simulation experiments on recommendation systems datasets} as well as on hand designed problem instances in Section~\ref{sec:experiments}. In Section~\ref{sec:related_work}, we contextualize MAB-CS, and our contributions in particular, into the broader multi-objective MAB literature. Finally in Section~\ref{sec:conclusions}, we summarize our contributions and chart the course for future research.

\section{Theoretical Framework and Lower Bounds}
\label{sec:theory_lower_bounds}

In this section we first formalize cost and quality regret and formulate them in terms of the expected samples of sub-optimal arms. Then, we present theoretical lower bounds on the expected number of sub-optimal arm samples incurred by any policy.

\textbf{Cost and Quality Regret:} Under the stationary stochastic bandit setting \citep{auer2002finite}, at time $t$ we observe scalar reward $r_t \geq 0$ by sampling arm $k_t$. Further $\Exp \brk*{r_t} = \mu_{k_t}$, and $r_t$ is distributed per the stationary distribution $\nu_{k_t}$. In the MAB framework a policy $\pi$ maps observation histories to a decision. For a $K$-armed bandit instance with arms $\mathcal{A}$, the policy represents the mapping $\pi: \mathcal{H}_t \to \mathcal{A}$, where $\mathcal{H}_t \coloneq \crl*{k_1, r_1, \ldots, k_{t-1}, r_{t-1}}$ denotes the complete observation history. The goal in MAB-CS is to design $\pi$ that converges onto sampling $a^*$ in a manner that balances between the sampling of feasible arms and accrual of least cost. These objectives are captured respectively by cost regret and quality regret. They are defined in Equation~\ref{eq:regret_definitions} for a problem instance $\nu$. 
\begin{align}
    \costRegret
    \prn{T, \nu, \pi}
    =
    \sum_{t = 1}^{T}
    \Exp_{\pi}
    \brk{
    \prn{
    c_{k_t}
    -
    c_{a^*}
    }^{+}
    }
    \,
    ,
    \,\,
    %
    \qualityRegret
    \prn{ T, \nu, \pi }
    =
    \sum_{t = 1}^{T}
    \Exp_{\pi}
    \brk{
    \prn{
    \mu_{\CS}
    -
    \mu_{k_t}
    }^{+}
    }
    .
    \label{eq:regret_definitions}
\end{align}
Here, $x^+ \coloneq \max \crl*{x, 0}$ is the zero-clipping operator, $k_t$ is a random variable denoting the policy decisions, and $\Exp_{\pi}$ denotes expectation over bandit arm choices. The zero-clipping ensures that the incremental regret from any policy decision is non-negative. Notably, only the sampling of the optimal arm $a^*$ has both zero incremental cost and quality regret. In practice, the zero-clipping of incremental quality regret $\prn{\mu_{\CS} - \mu_{k_t}}$ penalizes only violations of the feasibility constraint while ignoring reward above $\mu_{\CS}$. Since the goal is to converge onto sampling optimal arm $a^*$, cost regret with zero-clipping only penalizes samples of arms more expensive than $a^*$ while overlooking the cost accrual from cheaper arms. We denote the incremental cost and quality gaps from sampling arm $a_i$ by $\Delta_{C, i} \coloneq c_i - c_{a^*}$ and $\Delta_{Q, i} \coloneq \mu_{\CS} - \mu_{i}$ respectively. Using a standard regret decomposition result, in Equation~\ref{eq:regret_decomposition} we rewrite regret in terms of the expected number of samples of sub-optimal arms.   
\begin{align}
    \costRegret
    \prn{T, \nu, \pi}
    =
    \sum_{i = 1}^{K}
    \Delta_{C, i}^+
    \Exp \brk{ n_i \prn{T} }
    \,
    ,
    \quad
    \qualityRegret
    \prn{T, \nu, \pi}
    =
    \sum_{i=1}^{K}
    \Delta_{Q, i}^+
    \Exp \brk{ n_i \prn{T} }
    .
    \label{eq:regret_decomposition}
\end{align}
The reformulation in Equation~\ref{eq:regret_decomposition} permits the regret analysis of any policy $\pi$ by simply bounding the expected number of samples of sub-optimal arms. Moreover, lower bounds on expected samples imply lower bounds on regret under the decomposition. For the remainder of the paper we assume without loss of generality that bandit arms are indexed in non-decreasing order of cost. Further we characterize the sub-optimal arms as cheap and expensive per Definition~\ref{def:cheap_and_expensive_arms}
\begin{definition}[Cheap and expensive arms]
We define {\em cheap arms} $\mathcal{A}^{-}$ to be arms with cost lower than $c_{a^*}$: $\mathcal{A}^{-} \coloneqq \crl{ a_k \in \mathcal{A} \mid k < a^*}$. And {\em expensive arms} to have cost higher than $c_{a^*}$: $\mathcal{A}^{+} \coloneqq \crl{ a_k \in \mathcal{A} \mid k > a^* }$.
\label{def:cheap_and_expensive_arms}
\end{definition}

\textbf{Lower Bounds for MAB-CS:} This section discusses problem instance-dependent lower bounds on the expected number of samples of sub-optimal arms for the class of consistent policies. The restriction imposed by consistency is available in Appendix~\ref{sec:preliminaries} and amounts to a technical condition to exclude policies that perform unevenly well on certain instances while incurring high regret on others. In addition to informing us about the least samples accrued by a policy for an arm, these lower bounds serve to elucidate the challenges that need to be overcome by any policy that works to minimize regret. The complete proofs of our lower bounds are available in Appendix~\ref{sec:lb_proofs} and the current section provides proof sketches.

Results from information theory relate the expected number of samples drawn by a policy for two bandit instances with overlapping arms $\mathcal{A}$, but different arm reward distributions \citep{garivier2019explore}. The origin of sample lower bounds then lies in the construction of such an instance pair that differ minimally in their reward distributions but have distinct optimal arms. In the remainder of this section, we first discuss lower bounds for cheap arms $a_k \in \mathcal{A}^{-}$ and expensive arms $a_k \in \mathcal{A}^{+}$\footnote{The individual lower bounds for cheap arms, Theorem~\ref{thm:lb_low_cost}, and expensive arms, Theorem~\ref{thm:lb_high_cost}, have appeared in prior work \citep{juneja2025pairwise} and are re-stated for completeness of presentation.}. Finally a joint lower bound on the samples of arms whose expected rewards satisfy a certain constraint is presented\footnote{While the results of Section~\ref{sec:theory_lower_bounds} are for arm reward distributions that are Gaussian with unit variance, the generalization to other distribution families such as Bernoulli is straightforward \citep{lattimore2020bandit}.}. A proof sketch for Theorem~\ref{thm:lb_high_cost} builds intuition ahead of the more involved proof sketch for Theorem~\ref{thm:lb_joint}.  
\begin{figure}[ht]
    \centering
\includegraphics[width=0.99\linewidth]{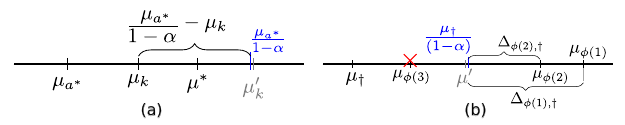}
    \caption{(a) Arm $a^*$ must be deemed feasible by each expensive arm $a_k \in \mathcal{A}^{+}$ introducing dependence on $\mu_{a^*} / (1 - \alpha) - \mu_k$. (b) Each arm in $\mathcal{A}^{\dagger}$ is diminished to make $\mu_{\dagger}$ feasible. $\phi(i)$ denotes the index of the $i^{\text{th}}$ highest reward arm; $a_{\phi(1)}$ and $a_{\phi(2)}$ are in $\mathcal{A}^{\dagger}$ while $a_{\phi(3)}$ is not.}
    \label{fig:lb_main_paper}
\end{figure}
\begin{theorem}[Lower Bound for Cheap Arms]
Expected number of samples for cheap arm $a_k$ satisfy,
\begin{align*}
&\liminf_{T \to \infty}\frac{\Exp
\brk{n_k \prn{T}}}{\log T}
\geq
\frac{2 }{ \Delta_{Q, k}^2 }
,
\,
\forall
\,
a_k 
\in 
\mathcal{A}^{-}
.
\end{align*}
\label{thm:lb_low_cost}
\end{theorem}
\begin{theorem}[Lower Bound for Expensive Arms]
Expected number of samples for expensive arm $a_k$ satisfy, 
\begin{align*}
\liminf_{T \to \infty}
\frac{\Exp \brk{n_k \prn{T}}}{\log T}
\geq
\frac{2 (1 - \alpha)^2}
{ \prn{
{
\displaystyle \mu_{a^*} 
} 
- 
(1 - \alpha) \mu_k
}^2 
}, \forall \,\, a_k \in \mathcal{A}^+
.
\end{align*}
\label{thm:lb_high_cost}
\end{theorem}
\textbf{Proof sketch for Theorem~\ref{thm:lb_high_cost}:} The proof for Theorem~\ref{thm:lb_high_cost} can be illustrated using Fig.~\ref{fig:lb_main_paper}(a). We construct a perturbed bandit instance $\nu'$ such that the expected reward $\mu_k'$ of the expensive arm $a_k \in \mathcal{A}^+$ under consideration is enhanced to be more than $\mu_{a^*}/(1 - \alpha)$, while all others are held constant. The result is that $a^*$ is no longer optimal nor feasible in instance $\nu'$ giving rise to Theorem~\ref{thm:lb_high_cost}. Next, an analogous sketch reveals a joint lower bound on the expected samples of certain high reward arms that can \textbf{eliminate} infeasible arms through pairwise comparisons.
\begin{definition}[Best reward cheap arm $a_{\dagger}$]
$\mu_{\dagger} \coloneqq \max_{a_i \in \mathcal{A}^{-}} \mu_i$ is the highest reward among cheap arms and is borne by arm $a_{\dagger}$.
\label{def:a_dagger}
\end{definition}
\begin{theorem}[Joint Lower Bound]
Expected number of samples of arms in $\mathcal{A}^{\dagger}$ satisfy, 
\begin{align*}
\liminf_{T \to \infty}
\frac{
{
\displaystyle
\sum_{ 
a_i \in \mathcal{A}^{\dagger} 
}
}
\Delta_{i, \dagger}^2
\,
\Exp \brk{n_i \prn{T}} 
}
{\log T}
\geq
2(1 - \alpha)^2
.
\end{align*}
Where $\Delta_{k, \dagger} \coloneqq (1 - \alpha) \mu_k - \mu_{\dagger}$ represent the gap between the subsidized reward of an arbitrary arm $a_k$ and $\mu_{\dagger}$. $\mathcal{A}^{\dagger} \coloneqq \crl{a_k \in \mathcal{A} \mid \Delta_{k, \dagger} > 0}$ represents arms with reward sufficient to eliminate $a_{\dagger}$.
\label{thm:lb_joint}
\end{theorem}
\textbf{Proof sketch for Theorem~\ref{thm:lb_joint}: } Intuitively, $\mu_{\dagger}$ represents the reward that shall be most challenging to deem infeasible. The set $\mathcal{A}^{\dagger}$ represents arms that are capable of eliminating $a_{\dagger}$ by virtue of their rewards being higher than $\mu_{\dagger} / (1 - \alpha)$, i.e. the inflated $\mu_{\dagger}$. More concretely, in Fig.~\ref{fig:lb_main_paper}{b}, we construct perturbed instance $\nu'$ by reducing the reward of all the arms lying to the right of inflated $\mu_{\dagger}$ to below $\mu_{\dagger} / (1 - \alpha)$. Since arm $a_{\dagger} \in \mathcal{A}^-$ is cheaper than $a^*$, the role of optimal arm is usurped by $a_{\dagger}$ in $\nu'$. Perturbing $\nu$ in this manner originates Theorem~\ref{thm:lb_joint} which is a strict generalization of a bound from prior work for the case when $\mathcal{A}^{\dagger}$ only contains $i^*$ \citep{juneja2025pairwise}.

\textbf{Takeaways from lower bounds:} Collectively the lower bounds inform us of the differences in expected rewards that must be resolved in MAB-CS and motivate the design of our algorithmic approach. Theorem~\ref{thm:lb_low_cost} informs us that any consistent policy must check the feasibility of each cheap arm separately against $\mu_{\CS}$. Since the position of $a^*$ in the cost-ordered line-up is not known apriori, this structure motivates an episodic and cost-ordered check on the feasibility of arms. Theorem~\ref{thm:lb_high_cost} exposes that a minimal requirement on the samples of an expensive arm $a_k \in \mathcal{A}^{+}$ will come from a comparison between $(1 - \alpha)\mu_k$ and the expected reward of the optimal arm $\mu_{a^*}$. Finally, Theorem~\ref{thm:lb_joint} reveals that the disqualification of an arm is solved optimally by combining together the samples of all arms that are capable of disqualifying it.

\section{Cost-Ordered Feasibility}
\label{sec:cof}

\begin{algorithm}[t]
    \caption{\textbf{C}ost \textbf{O}rdered \textbf{F}easibility (COF)}
    \label{algo:COF}
    \begin{algorithmic}[1]
    \item[] \textbf{Inputs:} $K$ armed bandit instance $\mathcal{A}=\{a_1, \ldots, a_K\}$ indexed in non-decreasing order of their costs $c_1 \le \cdots \le c_K$; Horizon $T$; Subsidy factor $\alpha$; Error tolerance $\delta$.
    \item[] \textbf{Initialize:} For each $a_k\in\mathcal{A}$ set $n_k = 0$, $\hat{\mu}_k = 0, \UCB_k = 0, \LCB_k = 0$; Candidate arm $a_\ell = a_1$.
    
    \State \textbf{for} $a_k \in \mathcal{A}$ \textbf{: } Sample $a_k$; update $n_k$, $\hat{\mu}_k$, $\UCB_k$, $\LCB_k$
    
    \While{{\small ${\displaystyle \sum_{a_k\in\mathcal{A}}} \! n_k < T$} }
        \State $\mathcal{G}_{\ell} \gets
            \crl{a_i \! \in \! \crl{ a_{\ell + 1}, \ldots, a_K } \! \mid \! { \scalebox{0.8}{($1-\alpha)$} } \UCB_i \geq \LCB_{\ell} }$
    
        \If{$\mathcal{G}_{\ell} = \emptyset$}
            \State Deem $a_\ell$ as optimal; sample $a_\ell$ thereafter 
        \Else
            \For{each $a_k \in \mathcal{A}$}
                \State $\epsilon_{k, \ell} \gets \epsilon \left(  n_k, \hat{\mu}_k, \UCB_{\ell}, \alpha \right)$ \Comment{Computed as per Expression~\ref{eq:eps_def}}
            \EndFor
    
            \If{${\displaystyle \prod_{a_k \in \mathcal{A}}} \epsilon_{k, \ell} \leq \delta$}
                \Comment{Combining (aggregating) samples}
                \State Deem $a_{\ell}$ infeasible. Next evaluate $a_{\ell + 1}$
            \Else
                \If{ $n_{\ell} < {\displaystyle \max_{a_i \in \mathcal{G}_\ell}} n_i$ }
                    \State Sample $a_\ell$; update $n_\ell$, $\hat{\mu}_\ell, \UCB_\ell, \LCB_\ell$
                    \Comment{Exclusive sampling}
                \Else
                    \State Sample $a_\ell$; update $n_\ell$, $\hat{\mu}_\ell, \UCB_\ell, \LCB_\ell$
                    \Comment{BAI-filter}
                    \\ 
                    \textbf{and} sample $a_i \! \in \! \mathcal{G}_\ell | \UCB_i \! > \! {\displaystyle \max_{a_j\in\mathcal{G}_\ell} } \LCB_j$;
                    update \! $n_i, \hat{\mu}_i, \UCB_i, \LCB_i \, \forall$ sampled $a_i$
                \EndIf        
            \EndIf
        \EndIf
    \EndWhile
    \end{algorithmic}
\end{algorithm}

The pseudocode for our Cost Ordered Feasibility (COF) algorithm is presented in Algorithm~\ref{algo:COF}. Our algorithm accepts as input a $K$-armed MAB instance with collection of arms $\mathcal{A}$. COF also requires as input the problem horizon $T$, and an error tolerance factor $\delta$. In practice, we tune $\delta$ using horizon $T$ to secure the regret guarantee of Theorem~\ref{thm:cof_regret_ub}. COF evaluates the feasibility of candidate arm $a_{\ell}$ by pitting it against a set of gating arms $\mathcal{G}_{\ell}$. We introduce candidate arms to COF in the sequence in which they are indexed, i.e. in increasing (non-decreasing) order of their costs. Starting from the cheapest arm $a_1$, COF is designed to deem each arm as either feasible or infeasible with sufficient confidence. If determined feasible, we solely sample that arm thereafter. However, if infeasible we move on to the next cheapest arm.

\textbf{Confidence bound scheme:} Whenever an arm is sampled by COF, its sample count is incremented, and its empirical mean, $\UCB$ (Upper Confidence Bound), and $\LCB$ (Lower Confidence Bound) are updated to reflect the observed reward. For any arm $a_k$, $\hat{\mu}_k$ is the running sample mean of observed rewards, and the radius for the confidence bounds $\beta_k (\delta)$ is determined by the error tolerance level $\delta$ and is equal to $\sqrt{ \log \prn*{1  /  \delta}  / 2 n_k }$. $\UCB_k$ and $\LCB_k$ are then given by $\hat{\mu}_k + \beta_k (\delta)$ and $\hat{\mu}_k - \beta_k (\delta)$. Intuitively, $\prn{ \LCB_k, \UCB_k }$ represents a high probability interval for the true expected mean $\mu_k$ of arm $a_k$. In COF we design tests that check both the feasibility and infeasibility of candidate arm $a_\ell$ relative to the the quality constraint $(1 - \alpha) \mu^*$. The design of both COF's feasibility and infeasibility criteria use the insight that for two arbitrary arms if $\LCB_k > \UCB_m$, then $\mu_k > \mu_m$ with high confidence.

\textbf{Determining feasibility:} $a_\ell$ is determined to be \textbf{feasible} only once we find that its expected reward exceeds the $(1 - \alpha)$ subsidized reward of every arm more expensive than itself. This condition is checked while resetting gating arms $\mathcal{G}_\ell$ (Line 3). The set $\mathcal{G}_\ell$ represents arms whose subsidized rewards have not yet been shown to be worse than $\mu_\ell$ with sufficient confidence. We deem arm $a_\ell$ to be feasible once we find $\mathcal{G}_\ell$ to be empty (Line 4). Since cheaper candidates are evaluated prior to more expensive ones, the first candidate deemed feasible is treated as optimal and is sampled for the remaining time-slots (Line 5). Arms cheaper than $a_\ell$, that have already been deemed infeasible in prior episodes, are excluded from $\mathcal{G}_\ell$ since their low reward makes them non-informative about $\mu^*$. 
\begin{figure}[ht]
    \centering
    \includegraphics[width=0.99\linewidth]{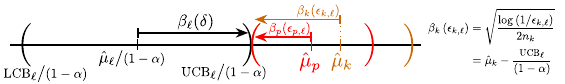}
    \caption{\textbf{Left:} The empirical mean of neither arm $a_p$ nor arm $a_k$ is sufficiently larger than $\UCB_{\ell} / (1 - \alpha)$ to eliminate $a_{\ell}$. Since $\hat{\mu} - \UCB_{\ell} / (1 - \alpha) < \beta_{\ell} (\delta)$ for both $a_p, a_k$. However if $\epsilon_{p, \ell}, \, \epsilon_{k, \ell}$ aggregated are less than $\delta$, then together they can eliminate $a_{\ell}$. \textbf{Right:} $\epsilon_{k, \ell}$ is derived using a back calculation.}
    \label{fig:eps_def}
\end{figure}

\textbf{Determining infeasibility:} Previous approaches have worked to first perform a best arm identification (BAI) round to lock in on $i^*$, and then proceed to make pairwise comparisons be introducing arms in cost-order \citep{juneja2025pairwise}. We forgo a BAI round which makes determining the infeasibility criteria more statistically challenging. Declaring infeasibility requires that candidate $a_{\ell}$ lose to the subsidized reward of any {\em single} arm. Moreover, as we understand from the discussion in Section~\ref{sec:theory_lower_bounds}, the infeasibility comparison is best made not by comparing candidate $a_\ell$ to any single arm, but rather by pooling together the confidence in the infeasibility of $a_\ell$ based on its comparisons with all arms. Corresponding to the comparison between $a_\ell$ and the subsidized reward of arm $a_k$ we compute an upper bound on the probability that $\mu_\ell > (1 - \alpha) \mu_k$ and denote it $\epsilon_{k, \ell}$ (Line 8). Once the product of $\epsilon_{k , \ell}$ across all comparisons drops below tolerance $\delta$, we deem $a_\ell$ to be infeasible (Lines 9-10). We reference this as the \textbf{{\em combining samples}} feature of COF. To compute $\epsilon_{k, \ell}$, we perform a back calculation, Fig.~\ref{fig:eps_def}, and obtain the formula stated in Expression~\ref{eq:eps_def} of Appendix~\ref{sec:expensive_arms_analysis}.

By line 10 of COF (Algorithm~\ref{algo:COF}), if $a_{\ell}$ cannot be conclusively deemed feasible or infeasible we continue to sample arm $a_{\ell}$ and all arms in $\mathcal{G}_{\ell}$ in lines 12-15. The sampling scheme of COF has a filter on $\mathcal{G}_{\ell}$ (line 15) built in. The filter on sampling arms in $\mathcal{G}_{\ell}$ pauses the sampling of arms that are with sufficient confidence not the best reward arm. This is done since the point of uniformly sampling all arms in $\mathcal{G}_{\ell}$ is to obtain an estimate of the feasibility threshold $\mu_{\CS} = (1 - \alpha) \mu^*$. However, if we find that the UCB of a gating arm's reward has fallen below the LCB of another gating arm, we can say with confidence that the arm with low UCB is not the best reward arm. We call this feature the \textbf{{\em BAI-filter}} and it leads to bandit arms having unequal samples.  

An arm whose sampling was paused by the BAI-filter may subsequently be evaluated for feasibility as a candidate feasible arm. If COF finds that the samples of candidate $a_{\ell}$ are trailing those of the most sampled arm in $\mathcal{G}_{\ell}$, COF will sample $a_{\ell}$ exclusively until $a_{\ell}$ catches up. We call this feature \textbf{{\em exclusive~sampling}} (Line 13) and its presence reduces the overall number of sub-optimal samples required by COF. The principle of sample matching for faster resolution between rewards has roots in probability theory and often appears in the MAB context when arms with wider confidence intervals are preferentially sampled \citep{gabillon2012best}. The {\em combining samples} and {\em exclusive sampling} featured by COF are deliberate design choices and their benefits on empirical performance are studied in Appendix~\ref{sec:experiment_details}.

\textbf{Expected regret:} Under normative progression, COF deems cheap arms $a_k \in \mathcal{A}^{-}$ infeasible during episodes leading up to episode $a^*$. Then, it deems $a^*$ feasible and samples it for the remaining time. Due to the formulation of cost and quality regret in terms of the expected number of samples in Equation~\ref{eq:regret_decomposition} we can bound COF's regret by upper bounding the expected number of suboptimal samples i.e. expected samples of cheap and expensive arms. We find that for cheap arm $a_{\ell}$ the sample upper bound can be traced to episode $\ell$ when its feasibility is evaluated, and for an expensive arm $a_k$, the bound emanates from $a_k$'s sampling during episodes $\dagger$ and $a^*$.
\begin{definition}[The quantities $\gamma^{\dagger}_k$ and $\gamma^{a^*}_{k}$]
For any expensive arm $a_k \in \mathcal{A}^+$ $\gamma^{\dagger}_k, \gamma^{a^*}_{k}$ are defined as,
\begin{align}
\gamma^{\dagger}_k
\coloneqq
\min
\crl*{
\frac{(3\sqrt{A} + 1)^2 \log T}{\sum_{i=1}^{A} \Delta_{\phi(i), \dagger}^2}
,
\frac{16 \log T }{\Delta_k^2}
}
,
\quad
\quad
\gamma^{a^*}_{k}
\coloneqq 
\displaystyle
\frac{16 \log T }
{
\prn*{\mu_{a^*} - (1 - \alpha) \mu_k}^2
}
\,\,
.
\end{align}
Where $\phi(k)$ represents the index of the $k^{\text{th}}$ highest reward arm, $\Delta_{\phi(k), \dagger} = (1 - \alpha) \mu_{\phi(k)} - \mu_{\dagger}$ is the gap between the subsidized reward of $a_{\phi(k)}$ and the reward of $a_{\dagger}$, and $A \leq \abs*{\mathcal{A}^{\dagger}}$ are the number of top-reward arms from $\mathcal{A}^{\dagger}$ that participate in deeming $a_{\dagger}$ infeasible during episode $\dagger$.
\label{def:gamma_definition}
\end{definition}
$\gamma^{\dagger}_k$ represents the bounding requirement from samples of $a_k$ during episode $\dagger$ and $\gamma^{a^*}_k$ represents the corresponding quantity during episode $a^*$, both under normative progression of COF. Inside the $\min$ operation for $\gamma^{\dagger}_k$ the first term is the number of times unfiltered arms in $G_{\ell}$ are sampled before $a_{\dagger}$ is deemed infeasible, and the second term is the number of samples beyond which sampling of $a_k$ is paused by the BAI-filter. The $\min$ reflects that the overall samples of arm $a_k$ during episode $\dagger$ will come from the earlier of these two terminations. $\gamma^{a^*}_k$ represents the bound on samples of $a_k$ from episode $a^*$. The BAI-filter does not play a role here since every arm in $\mathcal{G}_{\ell}$ must be eliminated by $a_{\ell}$ for $a^*$ to be deemed feasible. We use $\gamma^{\dagger}_k, \gamma^{a^*}_k$ to state the cost and quality regret bounds for COF.  
\begin{theorem}
For a bandit instance $\nu$ with arms $\mathcal{A}$ the expected cumulative cost and quality regret over horizon $T$ for COF are upper bounded as,  
\begin{align*}
\Exp \brk*{\costRegret \prn*{ T, \nu }}
&\leq
\sum_{a_k \in \mathcal{A}^+}
\max
\crl*{
\gamma^{\dagger}_k
,\,
\gamma^{a^*}_k
}
\Delta_{C, k}^+ 
+
K \sum_{a_k \in \mathcal{A}^+}
\Delta_{C, k}^+
\, 
,
\\
\Exp \brk*{\qualityRegret \prn*{ T, \nu }}
&\leq  
\sum_{a_k \in \mathcal{A}^-}
\frac{16 \log T}{\Delta_{Q, k}^+}
+
\sum_{a_k \in \mathcal{A}^+}
\max
\crl*{
\gamma^{\dagger}_k,
\,
\gamma^{a^*}_k
}
\Delta_{Q, k}^+
+
K \sum_{a_k \in \mathcal{A}} \Delta_{Q, k}^+ .
\end{align*}
\label{thm:cof_regret_ub}
\end{theorem}
\textbf{Order wise perspective on regret:} On studying Theorem~\ref{thm:cof_regret_ub} we find that while expensive arms $\mathcal{A}^{+}$ contribute to both cost and quality regret, cheap arms $\mathcal{A}^{-}$ only contribute to quality regret since $\Delta_{C, k}^+ = 0$ for them. The leading $O \prn*{\log T}$ term for cost regret stems from episodes $\dagger$ and $a^*$ and the overall contribution from expensive arm is the larger of the two. Both cheap and expensive arms may contribute to quality regret, accordingly two $O \prn*{\log T}$ terms appear in its upper bound. The trailing $O \prn*{1}$ terms in both cost and quality regret stem from the unlikely outcome that COF does not follow its normative trajectory.

\textbf{Comparison with lower bounds:} On dividing the leading $O (\log T)$ term for the cheap arms by the incremental quality regret $\Delta_{Q, k}$ it is apparent that the dependence of the expected samples on the reward gaps is $1 / \Delta_{Q, k}^2$. This is precisely the dependence in the lower bound of Theorem~\ref{thm:lb_low_cost} making COF {\em order-optimal for cheap arms}. The story on the optimality for expensive arms is more complex. The sample bound for expensive arms varies loosely as the $\max$ among $1 / \prn*{\mu_{a^*} - (1 - \alpha)\mu_k}^2$ and $1 / \sum_{i=1}^{A} \Delta_{\phi(i), \dagger}^2$. The former term has a tight match with the individual lower bound for expensive arms stated in Theorem~\ref{thm:lb_high_cost}, however the story with the latter term is not as simple. The joint lower bound of Theorem~\ref{thm:lb_joint} by itself allows for flexibility over the distribution of samples among arms in $\mathcal{A}^{\dagger}$. Imposing uniform sampling as a constraint in Theorem~\ref{thm:lb_joint}, we see the dependence  $1 / \sum_{i=1}^{|\mathcal{A}^{\dagger}|} \Delta_{\phi(i), \dagger}^2$ emerge. COF does not in general leverage all the arms in $\mathcal{A}^{\dagger}$ to eliminate $a_{\dagger}$. Moreover, for expensive arms the reciprocal sum of gap squared dependence is contingent on the arm being sampled to the end of episode $a_{\dagger}$. This need not happen as is revealed by the $\min$ in $\gamma^{\dagger}_k$ with the BAI-filter sample bound. Which is why for expensive arms there is a {\em mismatch} between the lower bound and COF's regret upper bound. This discrepancy arises in part from there being room for improvement in the sample lower bounds, particularly from incorporating the reward gaps $\Delta_k = \mu^* - \mu_k$. We construct an example to demonstrate this in Appendix~\ref{sec:lb_proofs}.

\textbf{Bounds compared to prior work:} When it comes to regret upper bounds, COF offers an improvement over the bounds proved for prior MAB-CS algorithms. ETC-CS is an algorithm proposed by \citet{pmlr-v130-sinha21a} that operates with a fixed exploration budget. Although order-optimal on worst-case regret, due to its exploration being non-adaptive, ETC-CS ends up with $O(T^{2/3})$ instance dependent regret; order-wise worse than COF. PE-CS does have $O (\log T)$ regret, and like COF, the algorithm is order-optimal for cheap arms \citep{juneja2025pairwise}. Due to different constant factors, a precise ordering between the COF and PE-CS bounds cannot be established. Therefore, we compare their bounds based on their dependence on reward gaps from the bandit instance. For expensive arms, PE-CS always accrues samples that vary as $1 / \Delta_{k}^2$. Moreover, for arm $i^*$ PE-CS has a regret upper-bound dependence on the smallest gap: $\max_{a_k \in \mathcal{A}} 1 / \Delta_k^2$. We improve on PE-CS in cases when COF can deem $a_{\dagger}$ infeasible sooner than the BAI-filter gets to $a_k$. When $\gamma^{a^*}_k$ dominates, we can see that the gap $\mu_{a^*} / (1 - \alpha) - \mu_k$ is a strictly larger gap than $\Delta_k = \mu^* - \mu_k$ making the instance-dependence of COF superior.

\section{Experiments}
\label{sec:experiments}

\begin{figure}[ht]
    \centering
    \includegraphics[width=0.99\linewidth]{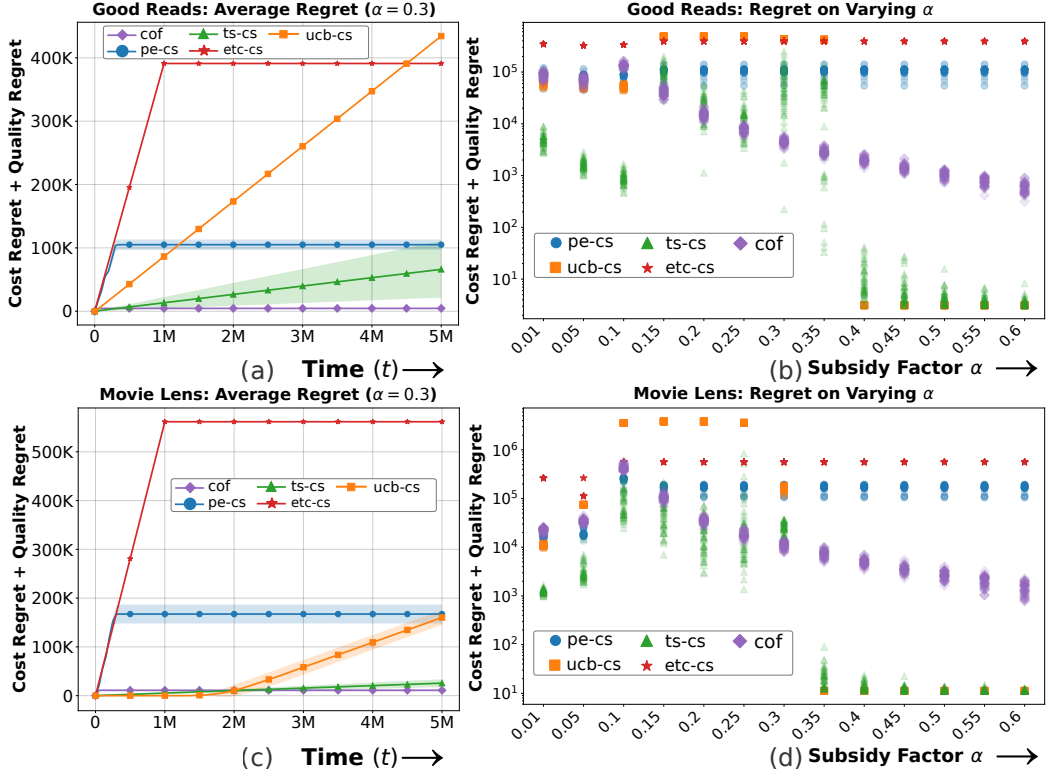}
    \caption{\textbf{Panels (a) and (c)} show the evolution of summed cost and quality regret over a horizon of 5 million samples for COF and other algorithms. Data is averaged over 50 runs and error bars represent the 20-80 percentile band. Left panels are for subsidy factor $\alpha=0.3$. \textbf{Panels (b), (d)} contain more comprehensive results for $\alpha$ varying between 0.01 and 0.60. Each $\alpha$ column has 50 markers per algorithm corresponding to 50 independent runs \textbf{Top:} Goodreads and \textbf{Bottom:} MovieLens.}
    \label{fig:gr_ml_vary_alpha}
\end{figure}
To validate COF empirically, we conduct experiments on the Goodreads \citep{wan2018item} and Movielens \citep{harper2016movielens} datasets from the recommendation systems literature. Goodreads contains crowd-sourced book ratings over 2.3 million unique books from 870,000 users. The books are organized into eight genres and each genre has between 36,514 and 335,449 books annotated with it. The number of aggregated reviews per genre are between 150,000 and 3.5 million. For MovieLens we use the 25M variant which consists of 25 million movie ratings over 62,000 movies rated by 162,000 users. Each movie is tagged with one or more genres with a total of 20 unique genres.

We simulate a scenario where a streaming service needs to make a decision over the genre of content served to their customer while minimizing the costs incurred from serving feasible quality content, a documented real-world objective \citep{aguiar2024platform, spotify_royalties_guide, zielnicki2025value}. Each genre is modeled as a bandit arm with a Bernoulli reward distribution and known cost. The reward represents the users quality of experience and the cost represents royalties paid out by the platform. The expected reward of an arm is computed as the average of ratings given to content annotated with a certain genre. The costs for our experiments are not native to the original datasets and are sampled uniformly at random between 0 and 1. Following this method we obtain Goodreads and MovieLens bandit instances detailed in Appendix~\ref{sec:experiment_details}. We simulate COF and four algorithms from the literature on these instances. These are PE-CS \citep{juneja2025pairwise}, UCB-CS, TS-CS, and ETC-CS \citep{pmlr-v130-sinha21a}. The results from the simulations on the Goodreads and MovieLens bandit instances are available in Fig.~\ref{fig:gr_ml_vary_alpha}.

Among the algorithms competing with COF, only \textbf{ETC-CS and PE-CS} have theoretical regret guarantees. \citet{pmlr-v130-sinha21a} do not offer precise criteria for selecting the exploration budget for ETC-CS so we pick it to be 20\% of the available horizon. Overall, from Fig.~\ref{fig:gr_ml_vary_alpha} we see that COF is better than both ETC-CS and PE-CS, consistent with our expectations from Section~\ref{sec:cof}. ETC-CS does not adapt to a problem instance and needs a conservatively chosen exploration budget. This leads to consistently high regret on both Goodreads and MovieLens. PE-CS explores in a an instance-adaptive manner, however it runs a best arm identification (BAI) routine to get the feasibility threshold $\mu_{\CS}$. Even as the feasibility criteria becomes looser with increasing subsidy factor $\alpha$, the terminal regret of PE-CS remains trapped on a high plateau. COF on the other hand does not suffer from such plateauing. For the smallest values of $\alpha$ we find that PE-CS outperforms COF by a small margin. Small alpha values in effect make accruing a large number of samples of the best reward arm $i^*$ necessary, so when it comes to this \textbf{small-alpha regime} PE-CS has a slight advantage over COF by virtue of tighter confidence intervals \citep{auer2010ucb}.

\textbf{UCB-CS and TS-CS} work by constructing a set of empirically feasible arms and then choosing to sample the cheapest arm from this set. These approaches lack guarantees on the correctness of the empirically feasible set and can in general include cheap arms in the set. While for some $\alpha$ values UCB-CS and TS-CS are successful at outperforming COF, their lack of reliability is evident in their outright failure on other scenarios. UCB-CS exhibits \textbf{linear regret} for several cases in both the Goodreads and MovieLens simulations. TS-CS fares better overall, however it too lacks reliability and has several outliers that drive a linear regret trend. For larger $\alpha$, 0.40 and beyond, UCB-CS and TS-CS fair well despite their unprincipled nature due to all arms being feasible. This means that their empirically feasible sets cannot accidentally include a cheap arm.

The regret vs. time trends of Fig.~\ref{fig:gr_ml_vary_alpha}(a) and (c) reveal how regret accrues over time. Optimal arm $a^*$ is the only arm that incurs zero incremental cost and quality regret. Regardless of the experiment or the specific independent run, a principled algorithm should start with an exploration phase where it accrues regret. Then, once the arm $a^*$ has been identified with sufficient confidence, the regret should flatten out. We see this for COF, PE-CS, and ETC-CS. However for TS-CS and UCB-CS we either find there to be a linear trend driven by all independent runs, UCB-CS on Goodreads (for $0.15 \leq \alpha \leq 0.35$), or by a subset of runs driving a less steep but still linear regret trend, TS-CS on Goodreads ($\alpha = 0.3, 0.35$) and MovieLens ($\alpha = 0.25$). On display is the fragility of UCB-CS. Once a cheap arm incorrectly makes into the empirically feasible set of UCB-CS, the algorithm will commence sampling it exclusively and not be able to recover from its mistake; Fig.~\ref{fig:gr_ml_vary_alpha}(c) starting about $t = 2$ million. Additional experiments demonstrating the \textbf{impact of ablating} the {\em exclusive-sampling} and {\em combining samples} features of COF are documented in Section~\ref{sec:experiment_details}.

\section{Related Work}
\label{sec:related_work}

Multi-armed bandits have a rich literature. In the following we first discuss cost-sensitive and multi-objective bandit frameworks. Then we highlight other problems where, similar to MAB-CS, the criteria for exploration are defined relative to unknown expected rewards in the bandit instance.\\ 
\textbf{Cost sensitive MABs:} Broadly, costs manifest as costs from {\em sampling} an arm \citep{badanidiyuru2018knapsacks, tran2012knapsack} or as costs caused by {\em observing} reward \citep{elumar2025probes}. The MAB-CS framework only incurs sampling costs that are deterministic and arm dependent. Most prior work from the sampling costs category \citep{badanidiyuru2018knapsacks, cayci2020budget, ding2013multiarmed, tran2010epsilon, tran2012knapsack, xia2015thompson, xia2016budgeted} deals with a fixed budget. However, budget constrained MABs are inappropriate for real-world decision-making settings where accrued cost simply has to be minimized while honoring a constraint on reward.\\
\textbf{Multi-objective problems:} Prior work by \citet{kanarios2024cost} minimizes accumulated cost regardless of budget, and works with the objective of identifying the best reward arm using the least cost. However, \citet{kanarios2024cost} have no consideration for feasible reward or regret. \citet{drugan2013designing} introduced the multi-objective framework of MO-MAB that can incorporate the dual objectives of cost and quality regret. However, their framework works to exploit any pareto-optimal bandit arm and does not provision a feasibility criteria like MAB-CS.\\
\textbf{Broader exploration objectives:} Examples of MAB frameworks where the usual goal of exploring to identify high reward arms is fused with another constraint include not dropping below baseline rewards \citep{du2021one, pmlr-v48-wu16}, identifying all $\epsilon$-optimal arms \citep{al2022complexity, mason2020finding}, and identifying top $m$ reward arms \citep{bubeck2013multiple, gabillon2012best, kalyanakrishnan2010efficient}. The statistical principles leveraged by these methods overlap with our work, especially probability aggregation \citep{kalyanakrishnan2010efficient}. However these methods are fundamentally different to MAB-CS since they lack the notion of cost. In our work we want to explore just enough to determine the cheapest feasible arm while being agnostic to the complete set of feasible arms. Our method prioritizes cheaper arms and only evaluates an arm's feasibility if cheaper arms have been deemed infeasible with high confidence.

\section{Conclusions and Future Work}
\label{sec:conclusions}

In this paper we worked on multi-armed bandits with cost subsidy (MAB-CS) with the minimum tolerated quality set to be a fraction of the unknown best reward. First, we developed lower bounds on the expected number of samples needed from sub-optimal arms. Then, we developed the COF algorithm to solve MAB-CS, analyzed its expected cost and quality regret, and validated COF through experiments. We found a discrepancy between the instance dependence of our joint sample lower bound, and COF’s corresponding upper bound, which we shall address in future research.

COF leverages a set of gating arms to determine reward feasibility and is constrained to uniformly sample the gating arms retained by the BAI-filter. Depending on the application it may be unwise to fixate on minimizing the number of sub-optimal samples and instead cost-biased sampling may be more suitable. Our joint lower bound result reveals flexibility on gating sample origins. If the highest reward arms are too expensive, COF should lean on more samples from less expensive arms of sufficient quality. A query contextual variant of COF with cost-biased sampling would be the appropriate scheme for LLM routers due to models having an uneven cost distribution with a large spread. We motivate the need for developing a cost-biased variants of COF using a synthetic example in Appendix~\ref{sec:experiment_details} and leave formal analysis and its contextual generalization to future work. 

\begin{ack}
\end{ack}

\bibliography{main}
\bibliographystyle{plainnat}

\newpage
\newgeometry{left=2cm, right=2cm, top=2.5cm, bottom=2.5cm}
\appendix
\section{Additional Experiments and Details}
\label{sec:experiment_details}

In this section we include all the empirical results and documentation supporting the experiments for reproducibility. All variants of COF were run by setting error tolerance parameter $\delta = \frac{K^2}{T^2}$ where $K$ is the number of bandit arms. Although this setting of $
\delta$ differs from the $\delta = \frac{1}{T^2}$, the theoretical bounds for our choice of $\delta$ are identical to those for $\delta = \frac{1}{T^2}$ and only differ on constant factors.

\subsection{Details on Main Paper Experiments}

\subsubsection*{Statement on the use of datasets in this paper}

We use the Goodreads Book Graph dataset released by Wan and McAuley and the MovieLens 25M dataset released by GroupLens. We cite the requested dataset publications for Goodreads \citep{wan2018item, wan2019fine} and MovieLens \citep{harper2016movielens}. Both datasets are used only for non-commercial academic research. We do not redistribute the raw datasets, and we comply with the stated usage terms: the Goodreads dataset is made available for academic use only and may not be redistributed or used commercially; the MovieLens 25M dataset may be used for research purposes with acknowledgment, without redistribution, without implying endorsement by the University of Minnesota or GroupLens, and without commercial or revenue-bearing use unless permission is obtained.

\subsubsection*{Instructions to recreate experiments}

We have shared a folder of supplemental materials with the structure:
\dirtree{%
.1 .
.1 data.
.1 scripts.
.1 source.
}

\begin{itemize}
\item \texttt{data} contains bandit instance data in plain text format
\item \texttt{scripts} contain bash files needed to generate all the log files and the plots
\item \texttt{source} contains the python source code implementing the algorithm and plotting logic.
\end{itemize}
Scripts contains the following bach scripts.
\dirtree{%
.1 .
.2 compress\_comb\_samp.sh.
.2 compress\_exclusive\_samp.sh.
.2 plot\_comb\_samp\_bespoke.sh.
.2 plot\_comb\_samp.sh.
.2 plot\_exclusive\_samp\_bespoke.sh.
.2 plot\_exclusive\_samp.sh.
.2 plot\_only.sh.
.2 plot\_regret\_vs\_time.sh.
.2 readme.md.
.2 run\_comb\_samp.sh.
.2 run\_exclusive\_samp.sh.
.2 run\_good\_reads\_ablation.sh.
.2 run\_good\_reads.sh.
.2 run\_movie\_lens\_ablation.sh.
.2 run\_movie\_lens.sh.
}
\begin{itemize}
    \item To generate Fig.~\ref{fig:gr_ml_vary_alpha}(a) perform the following sequence of steps:
    \begin{enumerate}
        \item Run \texttt{./scripts/run\_good\_reads.sh} from the project parent directory. All resource path based commands from here on out will be relative to the parent directory. This will save the log file into \texttt{results/run\_logs}
        \item Run \texttt{./scripts/plot\_regret\_vs\_time.sh} after making sure the correct destination of log files is hard coded into the bash file
    \end{enumerate}
    \item To generate Fig.~\ref{fig:gr_ml_vary_alpha}(b) perform the following sequence of steps:
        \begin{enumerate}
            \item Run \texttt{./scripts/run\_good\_reads.sh} from the project parent directory. All resource path based commands from here on out will be relative to the parent directory. This will save the log file into \texttt{results/run\_logs}
            \item Run \texttt{./scripts/plot\_only.sh} after making sure the correct destination of log files is hard coded into the bash file
        \end{enumerate}
        \item To generate Fig.~\ref{fig:gr_ml_vary_alpha}(c) perform the following sequence of steps:
        \begin{enumerate}
            \item Run \texttt{./scripts/run\_movie\_lens.sh} from the project parent directory. All resource path based commands from here on out will be relative to the parent directory. This will save the log file into \texttt{results/run\_logs}
            \item Run \texttt{./scripts/plot\_regret\_vs\_time.sh} after making sure the correct destination of log files is hard coded into the bash file
        \end{enumerate}
        \item To generate Fig.~\ref{fig:gr_ml_vary_alpha}(d) perform the following sequence of steps:
        \begin{enumerate}
            \item Run \texttt{./scripts/run\_movie\_lens.sh} from the project parent directory. All resource path based commands from here on out will be relative to the parent directory. This will save the log file into \texttt{results/run\_logs}
            \item Run \texttt{./scripts/plot\_only.sh} after making sure the correct destination of log files is hard coded into the bash file
        \end{enumerate}
\end{itemize}

The other bash scripts in \texttt{./scripts} are used to generate the plots available here in this supplemental experiments section.

\subsubsection*{Computational resources and wall-clock execution time}
To run our bandit experiments we used a machine with the below configuration:
\begin{table}[h]
\centering
\begin{tabular}{lllllllll}
\hline
OS & Access Methods & GPU & Cores & Memory (GB) & Scratch Storage (GB) \\
\hline
RHEL8 & ssh & NVidia Tesla T4 & 16 & 256 & 512 \\
\hline
\end{tabular}
\end{table}
Using this configuration, 50 independent runs of the tested MAB-CS algorithms for s single value of $\alpha$ takes between 5-10 minutes of wall clock time to run depending on the choice of algorithm. 

\begin{figure}[ht]
    \centering
\includegraphics[width=0.99\linewidth]{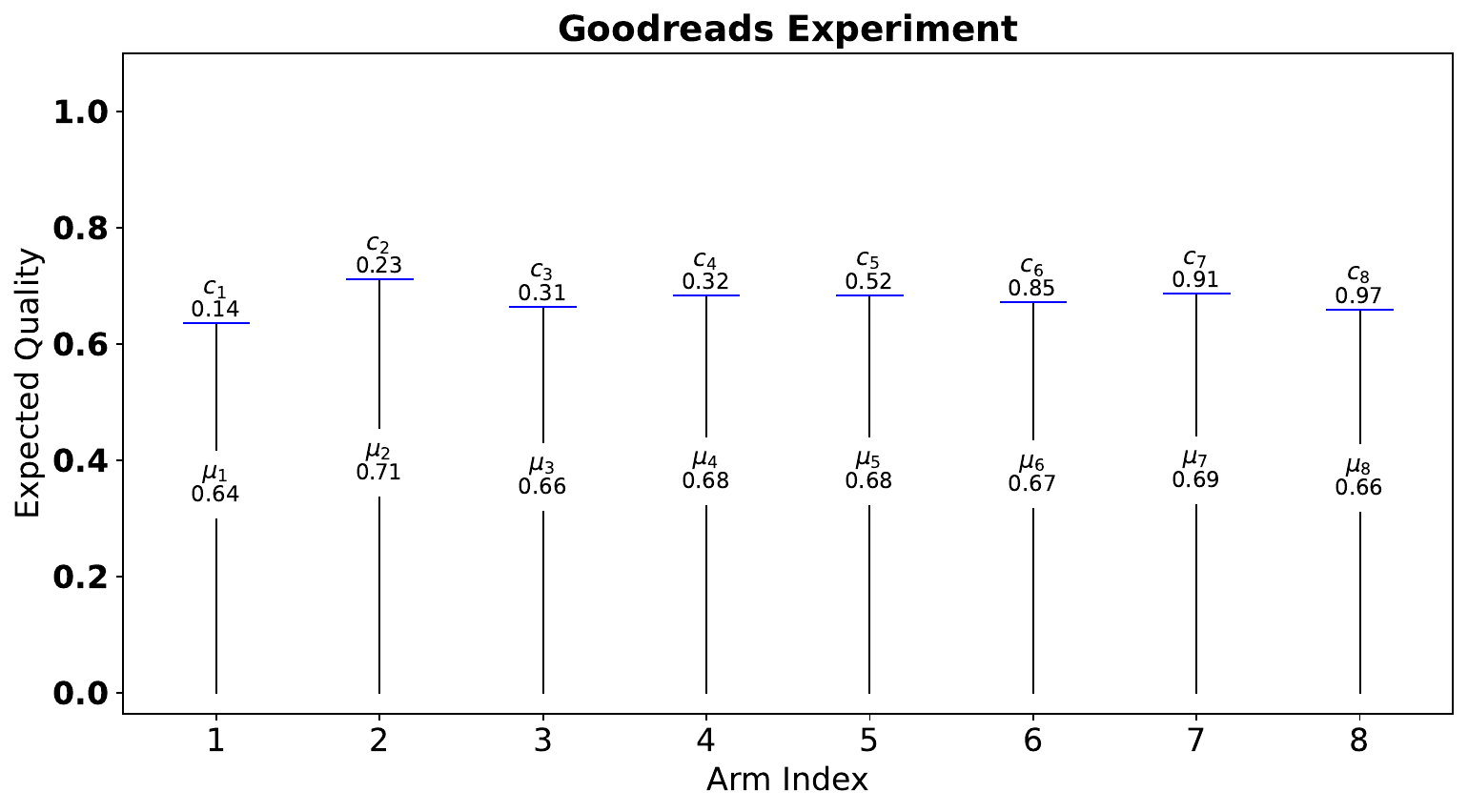}
    \caption{Good reads bandit instance}
    \label{fig:good_reads_instance}
\end{figure}

\begin{figure}[ht]
    \centering
\includegraphics[width=0.99\linewidth]{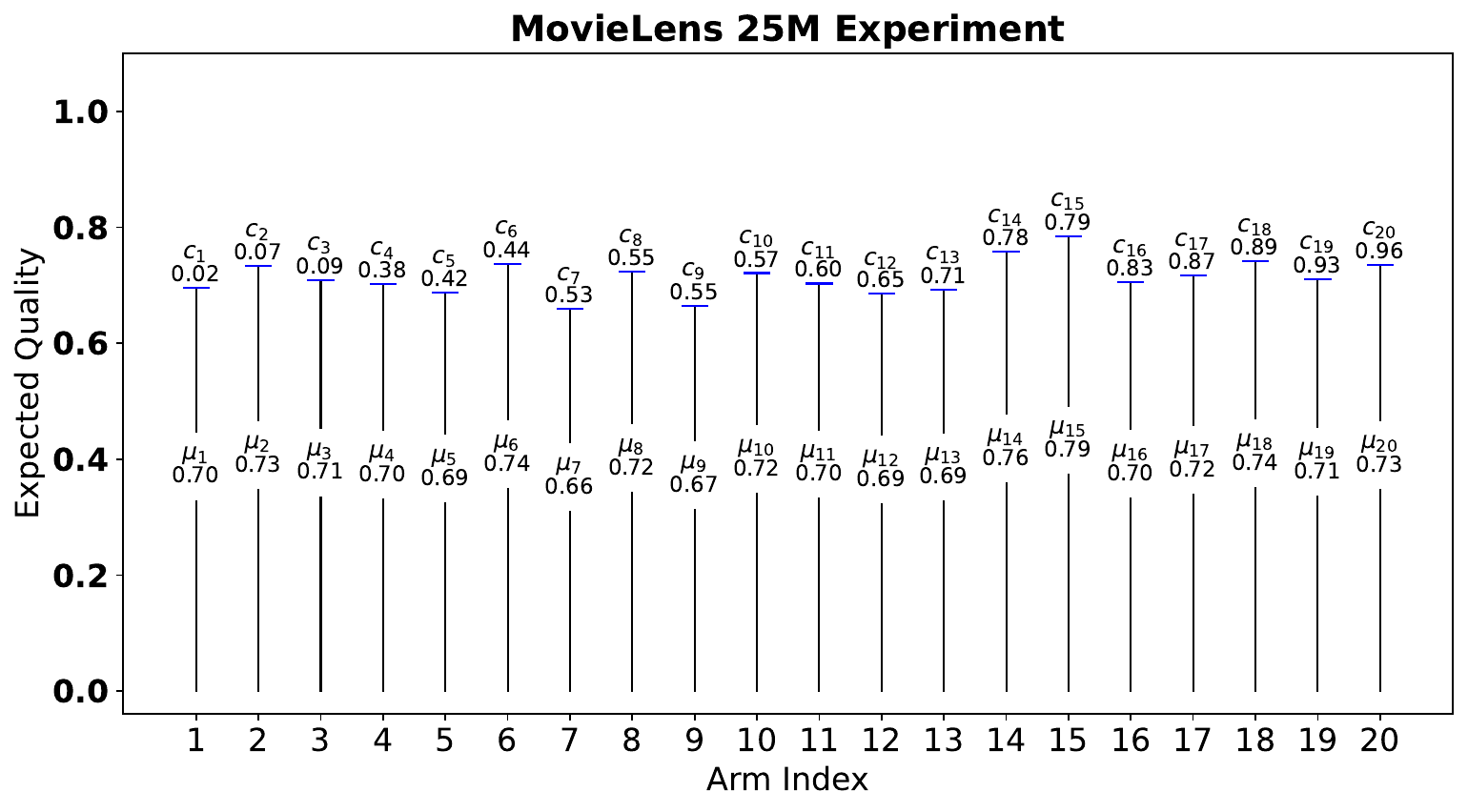}
    \caption{MovieLens bandit instance}
    \label{fig:movie_lens_instance}
\end{figure}

\subsection{Ablation Experiments}

\subsubsection*{Ablation instance $\nu_1$}

$\nu_1$ with arm reward array $\bm{\mu} = \brk*{0.15, 0.24, 0.96, 0.95, 0.99, 0.98, 0.97}$ and costs $1, 2, \ldots, 7$. We use $\alpha = 0.8$ making $\mu_{\CS} = 0.198$ and $a^* = 2$.

\subsubsection*{Ablation instance $\nu_2$}

$\nu_2$ with $\bm{\mu} = \brk*{0.44, 0.46, 0.48, 0.7, 0.71, 0.704, 0.714, 0.702, 0.716, 0.708, 0.712, 0.706}$ and costs $1, 2, \ldots, 12$. We use $\alpha = 0.3$ making  $\mu_{\CS} = 0.501$ and $a^* = 4$

\begin{figure}[ht]
    \centering
    \includegraphics[width=0.8\linewidth]{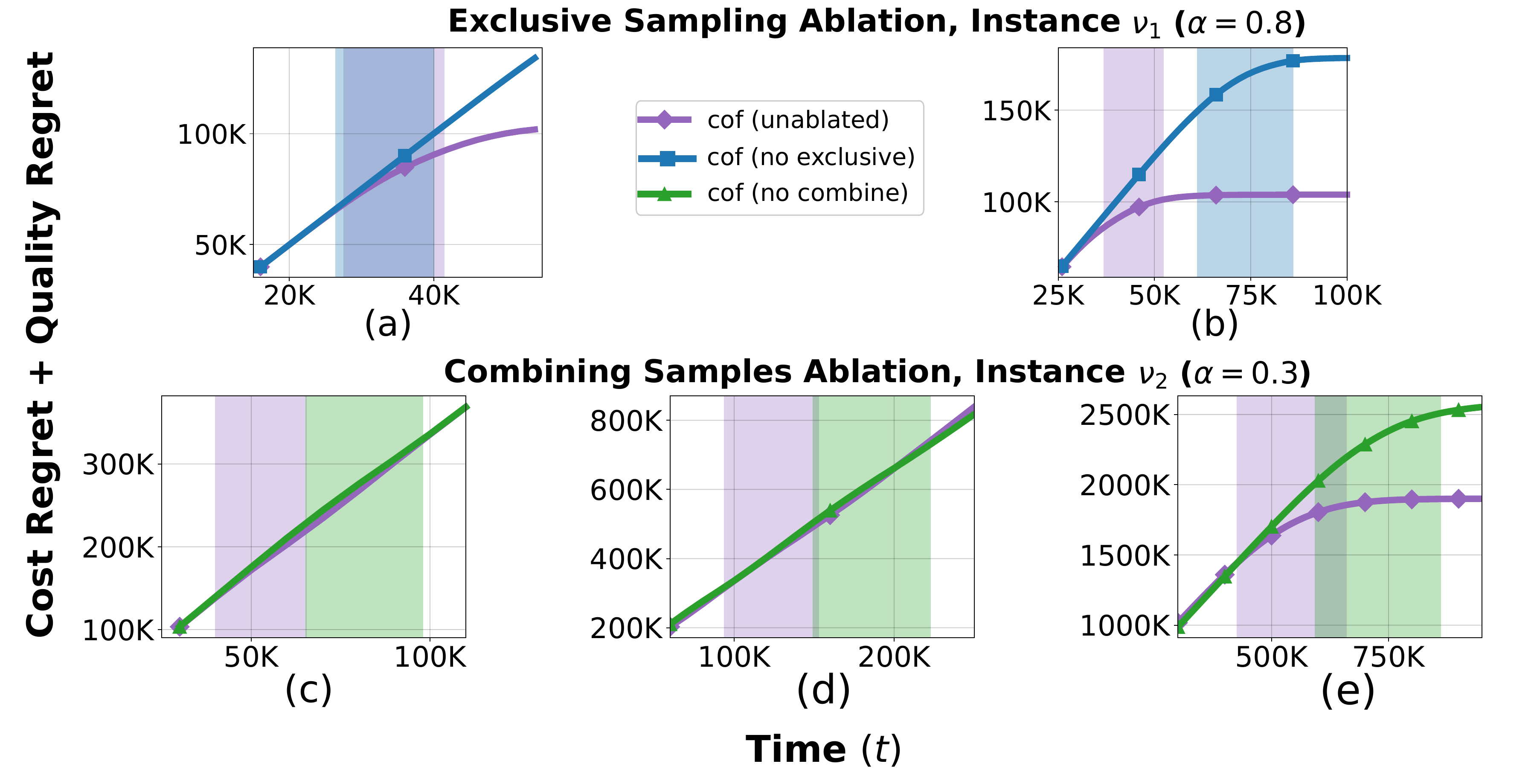}
    \caption{\textbf{Top:} Panels (a) and (b) depict windows from ablation on exclusive sampling experiment with instance $\nu_1$. Subsidy factor $\alpha = 0.8$, horizon $T = 200,000$ samples. The event rectangles in panels (a) and (b) represent the observed times at which $a_1$ was eliminated and $a_2$ was deemed optimal respectively. \textbf{Bottom:} Panels (c), (d), and (e) depict windows from ablation on combining samples experiment with instance $\nu_2$. Subsidy factor $\alpha = 0.3$, horizon $T = 1,000,000$. The event rectangles in panels (a), (b), and (c) represent the observed times at which arms $a_1$, $a_2$, and $a_3$ respectively were deemed infeasible. Data in figure is based on 1000 independent runs. All event rectangles capture one standard deviation of times at which the event occurred.}
    \label{fig:ablations}
\end{figure}

Next we highlight the role of the {\em exclusive sampling} and {\em combining samples} features of COF by ablating them from COF and comparing the performance of the COF to its ablated variants on cost and quality regret. Towards this end we define two variants of COF. First COF (no exclusive) always samples both the arm $a_{\ell}$ currently under evaluation for feasibility and arms in $\mathcal{G}_{\ell}$ (those not excluded by the BAI-filter as discussed in Section~\ref{sec:cof}). Second COF (no combine) does not aggregate the error probabilities (Line 9 of Algorithm~\ref{algo:COF}) and instead checks if any gating arm individually is capable of eliminating $a_{\ell}$. To highlight the role of exclusive sampling we design bandit instance where the second cheapest arm is optimal however is low reward leading to its sampling being paused early when it was a gating arm by the BAI-filter during the first episode that deems the cheapest arm infeasible. However a lot more of the samples of this second cheapest arm are needed to evaluate its own feasibility. We create 7-armed bandit instance $\nu_1$. Unablated COF has much better regret and the comparison of regret between COF and COF (no exclusive) is available in Fig.~\ref{fig:ablations} panels (a) and (b) which represent two zoomed in time windows.  

To illustrate the efficiency coming from combining samples (error probability aggregation) we choose a 12 armed bandit instance $\nu_2$. As seen from the results in Fig.~\ref{fig:ablations} panels (c)-(e) unablated COF is able to deem arms $a_1, a_2, a_3$ with fewer samples and achieve substantially lower regret than its ablated variant that does not combine samples of expensive arms. We highlight here that while {\em combining samples} will always lower both the cost and quality regret of COF, this is not true for {\em exclusive sampling}. One example of problem where fewer exclusive sampling leads to worse regret is when fewer total samples from a pair of arms, with one of the arms being the optimal $a^*$ comes at the cost of more samples from the arm that is sub-optimal among the two. However, we shall continue to have fewer total sub-optimal samples even in this example since without the catch-up caused by exclusive sampling, more samples of the sub-optimal arm will be needed to deem it infeasible.

\subsection{Non-uniform sampling with COF}

Experiment on cost-biased sampling. Bandit instance used was:
\begin{align*}
\bm{\mu} &= \brk*{0.38, 0.18, 0.50, 0.23, 0.44, 0.33, 0.28, 0.74, 0.80, 0.90, 0.96, 0.95, 0.48} \\
\bm{c} &= \brk*{1.0, 1.5, 2.0, 3.0, 4.0, 5.0, 6.0, 8.0, 10.0, 40.0, 80.0, 800.0, 850.0}
\end{align*}

\begin{figure}[ht]
    \centering
    \includegraphics[width=0.99\linewidth]{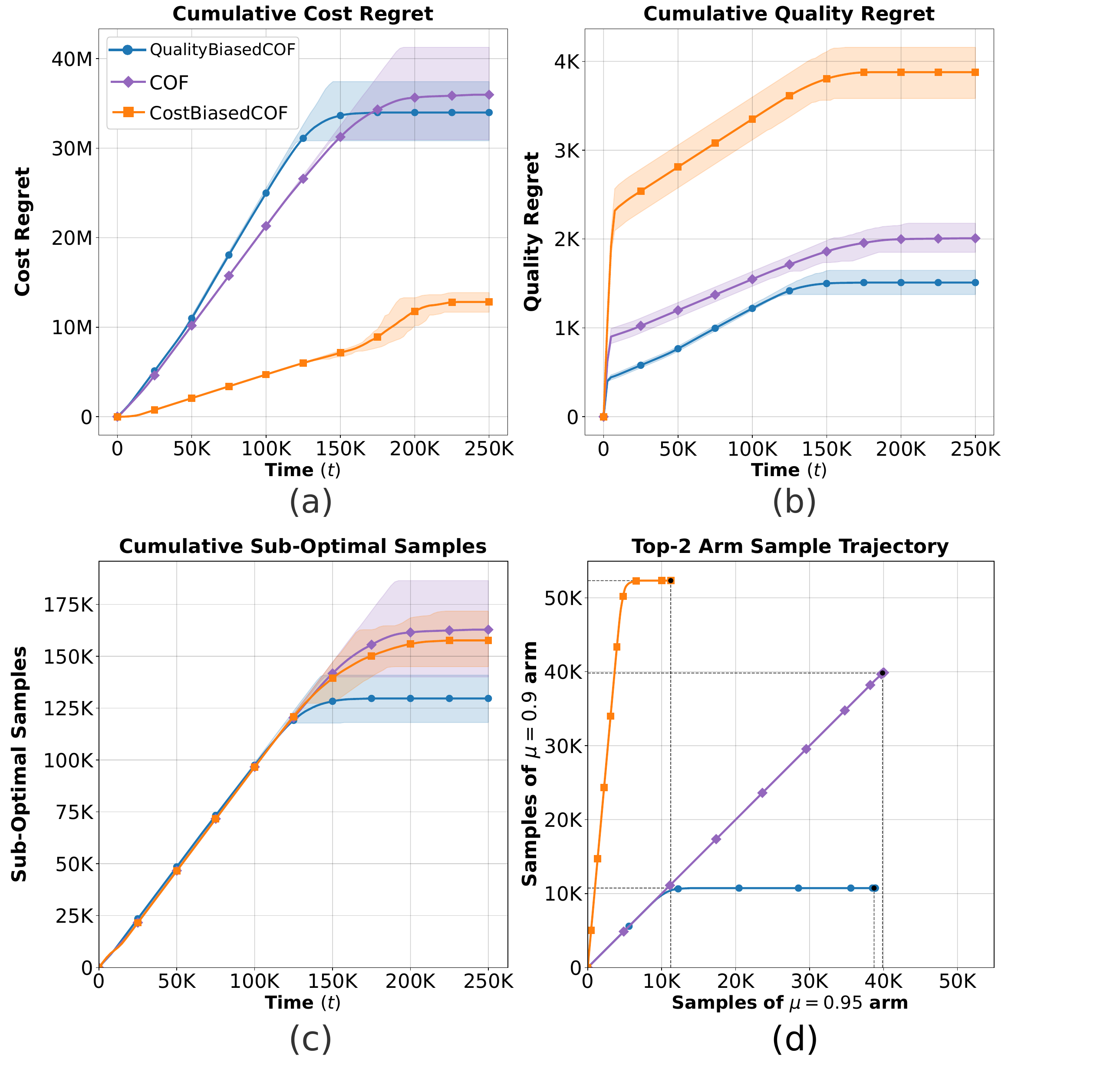}
    \caption{Variants of COF with uniform and non-uniform sampling}
    \label{fig:cof_non_uniform_sampling}
\end{figure}

\clearpage

\section{Mathematical Preliminaries}
\label{sec:preliminaries}

We present here some mathematical preliminaries that appear repeatedly in the proofs of Theorem's stated in the main paper. These results act as building blocks for the regret lower and upper bound analyses. Instead of proving results that are standard in the MAB literature, we provide a reference to a source that prove them. First, in Section~\ref{sec:lb_prelims} we state the tools leveraged in proving the lower bounds in Section~\ref{sec:lb_proofs}. Then, in Section~\ref{sec:cof_ub_prelims} we provide the results utilized in Section~\ref{sec:cof_ub} to analyze the upper bounds on the expected number of sub-optimal arms by COF (Algorithm~\ref{algo:COF}) and consequently the cost and quality regret incurred by COF. Throughout this Section, and for the analysis generally, we use the same notation as the main paper and introduce and reference new notation whenever necessary.

\subsection{Preliminaries for Lower Bound Analysis}
\label{sec:lb_prelims}

In this section we consider an MAB-CS instance with arms represented $\mathcal{A}$, and known costs of sampling each arm $a_i \in \mathcal{A}$. Just like in the main paper, arms are indexed in the non-decreasing order of their costs, and number of arms $\abs{ \mathcal{A} } = K$. An instance $\nu$ is completely specified as a collection of reward distributions $\crl{\nu_i}_{a_i \in \mathcal{A} }$, and each reward distribution $\nu_i$ is assumed to be drawn from some common distribution family $\mathcal{M}$, and $\Exp \brk{ \nu_i }$ is denoted $\mu_i$. The feasible arms $\mathcal{S} \subseteq \mathcal{A}$ are those whose expected reward $\mu_i \geq \mu_{\CS}$. Since $\mu_{\CS} = (1 - \alpha) \mu^*, \,\, \alpha \in [0, 1) $, means $\mathcal{S} \coloneqq \crl*{ a_i \in \mathcal{A} \mid \mu_i \geq (1 - \alpha) \mu^* }$. The optimal arm indexed $a^*$ is then the least cost feasible arm $a_{a^*} = \arg \min_{a_i \in \mathcal{S}} c_i$. The MAB-CS instance dependent lower bound on expected samples of sub-optimal arms is for the class of {\em consistent policies}.  

\begin{definition}[Consistent Policy $\pi$]
A policy $\pi$ is consistent if for all bandit instances $\nu$ and for all arms $i \neq a^*$,
$\Exp \brk{ n_i \prn{T} } = o(T^{\gamma})$ for all $0 < \gamma \leq 1$.
\label{def:consistent_policy}
\end{definition}

Next we state the fundamental inequality based on an information theory result used frequently in the MAB literature for characterizing problem instance dependent lower bounds on the expected number of samples of sub-optimal arms. 
\begin{lemma}[Fundamental inequality adapted from \citep{garivier2019explore}]
    Consider two MAB-CS instances $\nu$ and $\nu'$ over the same set of bandit arms $\mathcal{A}$. The reward distribution associated with sampling arm $a_i \in \mathcal{A}$ is $\nu_i$ in case of the first instance and $\nu_i'$ in case of the second. The expected number of samples $\Exp_{\nu} \brk{ n_i (T) }$ of an arm $a_i$ under any policy interacting with instance $\nu$ over $T$ sampling rounds must satisfy,  
    \begin{align}
        \sum_{a_i \in \mathcal{A}}
        \Exp_{\nu}
        \brk*{
        n_i (T)
        }
        \KL \prn{\nu_i, \nu_i'}
        &\geq
        \kl \prn*{\Exp_{\nu} \brk{Z}, \Exp_{\nu'} \brk{Z} }
        .
    \end{align}
    Where $\kl$ denotes the Kullback–Leibler (KL) divergence between two Bernoulli distributions, i.e.,
    \begin{align*}
        \forall \,\,
        p, q
        \in 
        \brk{0, 1}^2,
        \quad
        \kl \prn*{p, q}
        &= 
        p 
        \log \frac{p}{q}
        +
        (1 - p) 
        \log \frac{1 - p}{1 - q},
    \end{align*}
    and where $Z$ is any random variable with support lying in $\brk{0, 1}$. $Z$ must be measurable with respect to the probability space that all random variables are defined on.
    \label{lemma:lb_fundamental_inequality}
\end{lemma}

\begin{lemma}[Lower bound on $\kl \prn{p, q}$ from \citep{garivier2019explore}]
    The KL divergence between two Bernoulli distributions with parameters $p, q \in \brk{0, 1}^2$ is lower bounded as,
    \begin{align}
        \kl \prn*{p, q}
        &=
        p \log \frac{p}{q}
        +
        (1 - p) \log \frac{1 - p}{1 - q}
        \leq
        (1 - p) \log \frac{1}{1 - q} - \log 2
        .
    \end{align}
    \label{lemma:lb_on_bernoulli_kl}
\end{lemma}

Next we define some key quantities needed for the lower bound analysis,
\begin{definition}[Key quantities $D_{\inf}, \tilde{D}_{inf}$ from \citep{juneja2025pairwise}]
    Let $\KL$ denote the Kullback-Leibler divergence between two probability distributions. Given a distribution $\nu_i \in \mathcal{M}$ and real number $x$, we define
    \begin{align}
        D_{\inf} (\nu_i, x)
        = 
        \inf
        \crl*{
        \KL \prn{ \nu_i, \nu_i' }
        \mid
        \nu_i' \in \mathcal{M},
        \text{ and }
        \Exp \brk{ \nu_i' } > x
        },
        \\
        \tilde{D}_{\inf}
        \prn{ \nu_i, x }
        = 
        \inf
        \crl*{
        \KL \prn{\nu_i, \nu_i'}
        \mid
        \nu_i' \in \mathcal{M},
        \text{ and }
        \Exp \brk{\nu_i'} \leq x
        }.
    \end{align}
    \label{def:D_inf_definition}
\end{definition}
Where $\Exp \brk{ \nu_i' }$ denotes the mean of distribution $\nu_i'$. In the lower bound proofs of Section~\ref{sec:lb_proofs}, $D_{\inf}, \tilde{D}_{\inf}$ serve the role of quantifying the least perturbation to distributions $\crl*{ \nu_i }_{a_i \in \mathcal{A}}$ needed to induce a change in optimal arm $a^*$. We use $D_{\inf} (\nu_i, x)$ when $\mu_i < x$, and $\tilde{D}_{\inf} (\nu_i, x)$ when $\mu_i > x$.

\begin{lemma}[Explicit formulas for $D_{\inf}$, $\tilde{D}_{\inf}$ from \citep{lattimore2020bandit}]
    When the family of reward distributions $\mathcal{M}$ is all Gaussian distributions with a common variance $\sigma^2$, the key terms $D_{\inf}$ and $\tilde{D}_{\inf}$ take the form,
    \begin{align*}
        D_{\inf} \prn{ \nu_i, x }
        = 
        \tilde{D}_{\inf} \prn{ \nu_i, x }
        =
        \frac{\prn{ \mu_i - x }^2}{2 \sigma^2}
        .
    \end{align*}
     \label{lemma:d_inf_explicit_formula}
\end{lemma}

\subsection{Preliminaries for Upper Bound Analysis}
\label{sec:cof_ub_prelims}

\begin{definition}[Subgaussian random variable]
    \label{def:subgaussian}
    We say that $X$ is  $\sigma$-subgaussian if for any $\epsilon \geq 0$,
    \begin{align*}
        \Pr{ X - \Exp \brk{X} \geq \epsilon } &\leq 
        \exp{\prn{\frac{-\epsilon^2}{2 \sigma^2}}}.
    \end{align*}
\end{definition}

\begin{lemma}[Bounded random variables are Subgaussian, example 5.6(c) in \citep{lattimore2020bandit}]
    If Random Variable $ X \in \brk{a, b} $ almost surely, then $X$ is $\frac{b - a}{2}$ subgaussian.
    \label{lemma:bounded_is_subgaussian}
\end{lemma}

\begin{lemma}[Hoeffding Bound, Section 5.4 in \citep{lattimore2020bandit}]
    Let $X_1, X_2, \ldots, X_n$ be $n$ independent random variables, each bounded within the interval $[a, b]$ : $a \leq X_i \leq b$. The empirical mean of these variables is given by,
    \begin{align*}
        \Bar{X} &= \frac{1}{n} \sum_{i=1}^{n} X_i.
    \end{align*}
    Then Hoeffding's inequality states,
    \begin{align*}
        \Pr{ \Bar{X} - \Exp \brk{ X } \geq t } 
        &\leq 
        \exp \prn{ -\frac{2 n t^2}{ (b - a)^2 } },\\
        \Pr{ \Bar{X} - \Exp \brk{ X } \leq -t } 
        &\leq 
        \exp \prn{ -\frac{2 n t^2}{ (b - a)^2 } }.
    \end{align*}
    \label{lemma:hoeffding_bound}
\end{lemma}

\begin{corollary}[Hoeffding Bound for Bernoulli random variables]
    In the Hoeffding bound of Lemma \ref{lemma:hoeffding_bound}, when $\crl{X_i}_{i=1}^{n}$ are independent and identically distributed Bernoulli random variables with $\Exp \brk{X_i} = \mu$, then,
    \begin{align*}
        \Pr{ \hat{\mu} - \mu \geq t } 
        &\leq 
        \exp \prn{ -2 n t^2 },
        \\
        \Pr{ \hat{\mu} - \mu \leq -t } 
        &\leq
        \exp \prn{ -2 n t^2 }.
    \end{align*}
    Where $\hat{\mu} = \frac{1}{n} \sum_{i = 1}^n X_i$ denotes the sample mean over $n$ samples. 
    \label{corr:hoeffding_for_bernoulli}
\end{corollary}

\begin{lemma}[Probability of inaccurate $\UCB$ or $\LCB$ for Bernoulli rewards]
    For a bandit arm $a_i$ with reward distributed $\distBern{(\mu_i)}$, when the upper and lower confidence bound of the arm are defined with confidence radius $\beta_i (\delta) = \sqrt{\frac{\log (1 / \delta)}{2 n_i}}$ as $\UCB_i = \hat{\mu}_i + \beta_i (\delta)$, and $\LCB_i = \hat{\mu}_i - \beta_i (\delta)$, then we have,
    \begin{align*}
        \Pr{\UCB_i < \mu_i} &\leq \delta,\\
        \Pr{\LCB_i > \mu_i} &\leq \delta.
    \end{align*}
    \label{lemma:ucb_lcb_incorrect}
\end{lemma}
\begin{proof}
    Rearranging terms we can write,
    \begin{align*}
        \Pr{\UCB_i < \mu_i}
        &= \Pr{\hat{\mu}_i + \beta_{i} (\delta) < \mu_i}\\
        &= \Pr{\hat{\mu}_i - \mu_i < -\beta_i (\delta)}\\
        &\leq \exp \prn{- 2 n_i \beta^2_i (\delta)} = \delta.
    \end{align*}
    Where we have used the probability bound from Corollary \ref{corr:hoeffding_for_bernoulli}. The second part of the Lemma statement can be proved analogously.
\end{proof}

\begin{lemma}[Iterated expectation lemma]
Let $X$ be any integrable random variable over probability space $\prn*{ \Omega, \mathcal{F}, \mathbb{P} }$, and let $\crl*{E_i}_{i=1}^{n}$ be a collection of mutually exclusive and exhaustive measurable events. That is $\bigcup_{i=1}^{n} E_i = \Omega$ and $E_i \cap E_j = \phi, \, \forall \, i, j \in [n], i \neq j $. Then the following identity holds,
\begin{align*}
    \Exp
    \brk*{
        X
    }
    &=
    \sum_{i=1}^{n}
    \Exp
    \brk*{
        X \mid E_i
    }
    \Pr{E_i}
    .
\end{align*}
As a special case if the events are just some $E$ and its complement $E^c$, then,
\begin{align*}
    \Exp
    \brk*{
        X
    }
    &=
    \Exp
    \brk*{
        X \mid E
    }
    \Pr{E}
    +
    \Exp
    \brk*{
        X \mid E^c
    }
    \Pr{E^c}
    .
\end{align*}
\label{lemma:iterated_expectation_lemma}
\end{lemma}
\begin{proof}
    Define a sub $\sigma$-algebra of $\mathcal{F}$, $\mathcal{G} = \crl*{\phi, E_1, E_2, \ldots, E_n, \Omega}$. Then,
    \begin{align}
        \Exp
        \brk*{
            X
        }
        &=
        \Exp
        \brk*{
            \Exp
            \brk*{
                X
                \mid
                \mathcal{G}
            }
        }
        \text{ (Because $\mathcal{G} \subset \mathcal{F}$)}\\
        &=
        \sum_{i=1}^{n}
        \Exp
        \brk*{
            X \mid E_i
        }
        \Pr{E_i}.
    \end{align}
\end{proof}

\begin{lemma}[Expectation is at most equal to larger of the conditioned expectations]
    Let $X$ be any integrable random variable over probability space $\prn*{ \Omega, \mathcal{F}, \mathbb{P} }$, and let $\crl*{E_i}_{i=1}^{n}$ be a collection of mutually exclusive and exhaustive measurable events. That is $\bigcup_{i=1}^{n} E_i = \Omega$ and $E_i \cap E_j = \phi, \, \forall \, i, j \in [n], i \neq j $. Then,
    \begin{align}
        \Exp
        \brk*{
            X
        }
        &\leq
        \max_{i \in [n]} 
        \crl*
        {
        \Exp
            \brk*{
                X \mid E_i
        }
        }.
    \end{align}
    \label{lemma:expectation_less_than_max_lemma}
\end{lemma}
\begin{proof}
    Lemma \ref{lemma:expectation_less_than_max_lemma} can be considered a Corollary to Lemma \ref{lemma:iterated_expectation_lemma} as is illustrated by the following proof,
    \begin{align}
        \Exp
        \brk*{
            X
        }
        &=
        \sum_{i=1}^{n}
        \Exp
        \brk*{
            X
            \mid
            E_i
        }
        \Pr{ E_i }
        \text{ (From the proof of Lemma \ref{lemma:iterated_expectation_lemma}).}
        \\
        &\leq
        \prn*{
        \sum_{i=1}^n
        \Pr{ E_i }
        }
        \cdot
        \prn*{
        \max_{i \in [n]}
        \crl*{
        \Exp
        \brk*{
            X
            \mid
            E_i
        }
        }
        }
        \\
        &=
        \max_{i \in [n]}
        \crl*{
            \Exp
            \brk*{
                X
                \mid
                E_i
            }
        }.
    \end{align}
\end{proof}

In the analysis of the sample complexity of a sub-optimal arm $a_i \in \mathcal{A}, \, i \neq a^*$ incurred by COF (Algorithm~\ref{algo:COF}) we upper bound its expected samples $\Exp \brk*{n_i (T)}$ over horizon $T$ using Lemma~\ref{lemma:iterated_expectation_lemma} and conditioning on an event that ensures the normative progression of the algorithm. We formalize this templated approach through Lemma~\ref{lemma:bound_by_conditioning}
\begin{lemma}[Bounding expected samples through conditional expectation]
    For any event $G$ that is measurable with respect to the probability space that all reward random variables are defined on,
    \begin{align*}
        \Exp
        \brk{
        n_i (T)
        }
        &\leq
        \Exp
        \brk{
        n_i (T)
        \mid
        G
        }
        +
        T 
        \Pr{G^c}
        .
    \end{align*}
    \label{lemma:bound_by_conditioning}
\end{lemma}
\begin{proof}
    The proof follows trivially from the iterated expectation Lemma~\ref{lemma:iterated_expectation_lemma}, and the fact that the total number of samples of any arm can be at most $T$.
    \begin{align*}
        \Exp
        \brk*{ n_i (T) }
        &=
        \Exp
        \brk*{ n_i (T) \mid G} 
        \Pr{G}
        +
        \Exp
        \brk*{ n_i (T) \mid G^c}
        \Pr{G^c}
        \quad
        \text{ (Lemma~\ref{lemma:iterated_expectation_lemma})}
        \\
        &\leq
        \Exp
        \brk*{ n_i (T) \mid G} 
        +
        T
        \Pr{G^c}
        \quad
        \text{($\Pr{G} \leq 1$ and $n_i (T) \leq T$)}
        .
    \end{align*}
\end{proof}

\begin{lemma}
    For any collection of $N$ non-negative numbers $\crl{x_i}_{i = 1}^N$, $x_i \geq 0$, their root mean square is larger than their arithmetic mean,  
    \begin{align}
        \sqrt{
        \frac{\sum_{i = 1}^N x_i^2}{N}
        }
        &\geq
        \frac{\sum_{i = 1}^N x_i}{N}
    \end{align}
    \label{lemma:rms_geq_am}
\end{lemma}
\begin{proof}
    The proof follows from applying the Cauchy–Schwarz inequality to the $N$ length vectors $\mathbf{x} = \brk*{x_1, x_2, \ldots, x_N}$ and $\mathbf{1} = \brk*{1, 1, \ldots, 1}$ 
    \begin{align}
        \prn*{
        \sum_{i = 1}^{N} x_i^2
        }
        \cdot
        \prn*{
        \sum_{i = 1}^{N} 1^2
        }
        &\geq
        \prn*{
        \sum_{i = 1}^N
        x_i \cdot 1
        }^2
        \quad
        \text{ (Cauchy-Schwarz inequality)}
        \\
        \implies
        N
        \sum_{i = 1}^N
        x_i^2
        &\geq
        \prn*{
        \sum_{i = 1}^{N}
        x_i
        }^2
        \\
        \implies
        \frac{\sum_{i = 1}^N
        x_i^2}{N}
        &\geq
        \frac{1}{N^2}
        \prn*{
        \sum_{i = 1}^{N}
        x_i
        }^2
        .
  \end{align}
    Since each $x_i \geq 0$, taking principal square roots on both sides establishes the stated result.
\end{proof}

\begin{lemma}[Max-min Lemma]
    For any function $f: \reals^2 \to \reals$,
    \begin{align*}
        \max_{x} \min_{y} f (x, y)
        \leq
        \min_{y} \max_{x} f (x, y)
        .
    \end{align*}
    Provided the maximum and minimum are achievable.
    \label{lemma:max_min_lemma}
\end{lemma}
\begin{proof}
For all $x, y \in \reals$ by the definition of $\min, \max$,
\begin{align*}
    \min_u f(x, u)
    &\leq
    f(x, y)
    \leq
    \max_t f(t, y)
    \\
    \implies
    \min_u f(x, u)
    &\leq
    \min_s \max_t f(t, s)
    \\
    \implies
    \max_v \min_u f(v, u)
    &\leq
    \min_s \max_t f(t, s)
    .
\end{align*}
Where $u, v, s, t \in \reals$ are arbitrary variables. 
\end{proof}

\begin{lemma}[Regret Decomposition Lemma, Lemma 4.5 in \citep{lattimore2020bandit}]
    For any policy $\pi$ and stochastic bandit environment $\nu$ with $K$ arms, for horizon $T$, the Expected Cumulative Regret $\textrm{Reg}_{\pi} \prn*{T, \nu}$ of policy $\pi$ in $\nu$ satisfies, 
    \begin{align*}
        \Exp
        \brk*
        {
            \text{Reg}_{\pi} \prn*{T, \nu}
        }
        &=
        \sum_{i \in [K]}
        \Delta_i
        \Exp
        \brk*{
            n_i (T)
        }.
    \end{align*}
    This result may be trivially generalized to other notions of regret where the gap determining the incremental regret due to arm $i$ is some arbitrary $\Delta_{X, i}$. In this case, the regret decomposition shall be,
    \begin{align*}
        \Exp
        \brk*{
            \text{Reg}^X_{\pi} \prn*{T, \nu}
        }
        &=
        \sum_{i \in [K]}
        \Delta_{X, i}
        \Exp
        \brk*{
            n_i (T)
        }.
    \end{align*}
    In particular in our problem we have Cost and Quality regret which are,
    \begin{align*}
        \Exp \brk*{ \costRegret(T, \nu) }
        &= 
        \sum_{i \in [K]} \Delta^{+}_{C, i}
        \Exp \brk*{ n_i \prn*{T} }
        \\
        \Exp \brk*{ \qualityRegret(T, \nu) } 
        &= \sum_{i \in [K]} \Delta^{+}_{Q, i}
        \Exp \brk*{ n_i \prn*{T} }.
\end{align*}
\label{lemma:regret_decomposition_lemma}
\end{lemma}

\clearpage

\section{Lower Bound Proofs}
\label{sec:lb_proofs}

This section proves the MAB-CS lower bounds stated in Section~\ref{sec:theory_lower_bounds} for a consistent policy (Definition~\ref{def:consistent_policy}) on an MAB-CS instance with arms $a_i \in \mathcal{A}$ and reward distributions $\nu = \crl{ \nu_i' }_{a_i \in \mathcal{A}}$. For each proof we introduce a perturbed instance $\nu' = \crl{ \nu_i' }_{a_i \in \mathcal{A}}$.

The optimal arm is defined as usual $a_{a^*} = \arg \min_{a_i \in \mathcal{S}} c_i$ where $\mathcal{S} = \crl*{ a_i \in \mathcal{A} \mid \mu_i \geq (1 - \alpha) \mu^* }$ is the set of feasible arms. Instance $\nu'$ shall be constructed in a manner that its optimal arm differs from the optimal arm of instance $\nu$. In particular, in each proof we will use the fundamental inequality stated in Lemma~\ref{lemma:lb_fundamental_inequality} with $Z = \frac{ \Exp \brk{n_k (T)} }{T}$ representing the sample fraction of a certain arm $a_k \in \mathcal{A}$ that is sub-optimal for $\nu$ but optimal for $\nu'$.

\begin{figure}[ht]
    \centering
\includegraphics[width=0.7\linewidth]{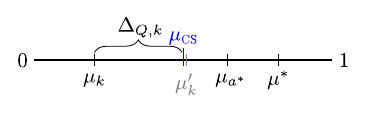}
    \caption{Cheap arms $a_k \in \mathcal{A}^-$ are infeasible as indicated by the gap $\Delta_{Q, k}$.}
    \label{fig:lb_cheap}
\end{figure}
We outline the proof of Theorem~\ref{thm:lb_low_cost} using the reward line illustration of Fig.~\ref{fig:lb_cheap}. From the original bandit instance $\nu$ with optimal arm $a^*$ that contains $a_k \in \mathcal{A}^{-}$, we construct a perturbed bandit instance $\nu'$ in which the expected reward $\mu_k'$ of $a_k$ is enhanced to be more than the feasibility threshold $\mu_{\CS}$, and all other expected rewards are held constant. Since $a_k$ is cheaper than $a^*$, the optimal arm in instance $\nu'$ shall be $a_k$. The lower bound of Theorem~\ref{thm:lb_low_cost} stems from a policy solving both instances $\nu$ and $\nu'$. No amount of reward enhancement can make an expensive arm optimal so their sample lower bound does not follow Theorem~\ref{thm:lb_low_cost}.

\begin{theorem}[Lower bound for samples of cheap arms]
    For a bandit instance $\nu$, over horizon $T$ for all consistent policies the expected number of samples of a cheap arm are lower bounded as,
    \begin{align}
        \liminf_{T \to \infty}
        \frac{\Exp_{\nu} \brk{n_i (T)} }{\log T}
        &\geq
        \frac{1}{D_{\inf} \prn*{\nu_i, (1 - \alpha)\mu^*} }
        \,\,
        \forall
        \,
        a_i \in \mathcal{A}^{-},
    \end{align}
    where $D_{\inf}$ is from Definition~\ref{def:D_inf_definition}, and $\mathcal{A}^-$ are the set of cheap arms with cost-ordered index less than $a^*$. 
    \label{thm:lb_low_cost_dinf}    
\end{theorem}
\begin{proof}
    Given bandit instance $\nu$, and cheap arm $a_i \in \mathcal{A}^-$ we construct instance $\nu'$ such that $\nu'_k = \nu_k \, \forall \, a_k \in \mathcal{A}, \, a_k \neq a_i$, and $\nu'_i \in \mathcal{M}$ is such that $\Exp \brk{ \nu_i' } = \mu_i \geq (1 - \alpha) \mu^*$. If $a^*$ denotes the index of the optimal arm for $\nu$, $c_i < c_{a^*}$, and arm $a_i$ is optimal for $\nu'$ since expected reward of $a_i$ is feasible and it is cheaper than $a^*$.

    Theorem~\ref{thm:lb_low_cost_dinf} can be shown by applying Lemma~\ref{lemma:lb_fundamental_inequality} to the instance pair $\nu, \nu'$ with $[0, 1]$ bounded random variable $Z = \dfrac{ \brk{ n_i (T) } }{ T }$ denoting the sample fraction of arm $a_i$ under any consistent policy.
    \begin{align}
        \sum_{a_k \in \mathcal{A}}
        \Exp_{\nu} \brk{ n_k (T) }
        \KL \prn*{ \nu_k, \nu_k' }
        &\geq
        \kl 
        \prn*{ \frac{ \Exp_{\nu} \brk*{ n_i (T) } }{T}, \frac{ \Exp_{\nu'} \brk*{ n_i (T) } }{T} }
        \nonumber
        \\
        \implies
        \Exp_{\nu} \brk{ n_i (T) } 
        \KL \prn*{ \nu_i, \nu_i' }
        &\geq
        \prn*{ 1 - \frac{ \Exp_{\nu} \brk*{n_i (T)} }{T} }
        \log 
        \prn*{ \frac{T}{ T - \Exp_{\nu'} \brk{ n_i (T) } } }
        - \log 2
        .
        \label{eq:initial_lb}
    \end{align}
    Where Equation~$\ref{eq:initial_lb}$ follows from instance $\nu$ and $\nu'$ being identical outside of arm $a_i$, and from the lower bound on the KL divergence between Bernoulli random variables (Lemma~\ref{lemma:lb_on_bernoulli_kl}). 

    For a consistent policy $\Exp_{\nu} \brk*{ n_i (T) } = o \prn{ T^\gamma } $, and $\Exp_{\nu'} \brk*{ n_i (T) } = T - o \prn{ T^\gamma } $, since total samples are $T$ and $a_i$ is optimal for $\nu'$. So,
    \begin{align*}
        \liminf_{ T \to \infty }
        \frac{1}{\log T}
        \prn*{ 1 - \frac{ \Exp_{\nu} \brk*{n_i (T)} }{T} }
        \log 
        \prn*{ \frac{T}{ T - \Exp_{\nu'} \brk{ n_i (T) } } }
        &\geq 
        1 - \gamma, 
        \, \forall \,
        \gamma \in (0, 1].
    \end{align*}
    Integrating bound in Equation~\ref{eq:initial_lb} with the limiting bound on its right hand side,
    \begin{align*}
        &\liminf_{T \to \infty}
        \frac{ \Exp \brk*{ n_i (T) } \KL \prn{\nu_i, \nu_i'} }{ \log T }
        \geq
        1
        \\
        \implies
        &\liminf_{T \to \infty}
        \frac{ \Exp_{\nu} \brk*{ n_i (T) } }{ \log T }
        \geq
        \frac{1}{ \KL \prn{ \nu_i, \nu_i' } }
        \,\, \forall \,
        \nu_i' \in \mathcal{M} \text{ such that } \Exp \brk{\nu_i'} = \mu_i \geq (1 - \alpha) \mu^*. 
    \end{align*}
    Therefore the tightest limiting lower bound is the one stated in Theorem~\ref{thm:lb_low_cost_dinf}.
\end{proof}
\begin{proof}[Proof of Theorem~\ref{thm:lb_low_cost}]
    The proof of Theorem~\ref{thm:lb_low_cost} follows in a straight forward manner from imposing the restriction that the distribution family $\mathcal{M}$ is all Gaussian distributions with unit variance on Theorem~\ref{thm:lb_low_cost_dinf}. Under the stated restriction on $\mathcal{M}$, $D_{\inf} \prn{ \nu_i, (1 - \alpha)\mu^* } = \frac{\prn{(1 - \alpha)\mu^* - \mu_i}^2}{2} = \Delta_{Q, i}^2 / 2 \,\, \forall \, a_i \in \mathcal{A}^-$. Where the explicit formula for $D_{\inf}$ is from Lemma~\ref{lemma:d_inf_explicit_formula}.
\end{proof}

\begin{theorem}[Lower bound for samples of expensive arms]
    For a bandit instance $\nu$, over horizon $T$ for all consistent policies the expected number of samples of an expensive arm are lower bounded as,
    \begin{align}
        \liminf_{T \to \infty}
        \frac{ \Exp_{\nu} \brk{n_i (T)} }{ \log T }
        &\geq
        \frac{1}{ D_{\inf} \prn*{ \nu_i, \frac{\mu_{a^*}}{1 - \alpha} } }
        \,\,
        \forall
        \,
        a_i \in \mathcal{A}^+,
    \end{align}
    where $D_{\inf}$ is from Definition~\ref{def:D_inf_definition}, and $\mathcal{A}^+$ are expensive arms with cost-ordered index more than $a^*$.
    \label{thm:lb_high_cost_dinf}    
\end{theorem}
\begin{proof}
    Along the lines of the proof of Theorem~\ref{thm:lb_low_cost_dinf}, to prove a lower bound on the expected samples accrued of expensive arm $a_i \in \mathcal{A}^+$ we construct $\nu'$ such that $\nu'_k = \nu_k \, \forall \, a_k \in \mathcal{A}, \, a_k \neq a_i$, and $\nu'_i \in \mathcal{M}$ is such that $\Exp \brk{ \nu_i' } = \mu_i > \frac{\mu_{a^*}}{1 - \alpha} \geq \mu^*$. Since arm $a_i$ is more expensive than arm $a_{a^*}$, the only mechanism by which a perturbation to $\nu_i$ can upend the optimality of $a_{a^*}$ is by raising the feasibility threshold $\mu_{\CS}$ so that $a_{a^*}$ is infeasible in instance $\nu'$. 
    
    Let arm $a_k$ denote the arm that is sub-optimal for $\nu$ but optimal for $\nu'$ and random variable $Z = \frac{ n_k (T)}{ T }, \, Z \in [0, \, 1]$ be the sample fraction of $a_k$ by any policy. Using the fundamental inequality of Lemma~\ref{lemma:lb_fundamental_inequality},
    \begin{align*}
        \Exp \brk*{ n_i (T) }
        \KL \prn{\nu_i, \nu_i'}
        &\geq
        \kl
        \prn*{ 
        \frac{ \Exp_{\nu} \brk*{ n_k (T) } }{ T },
        \frac{ \Exp_{\nu'} \brk*{ n_k (T) } }{ T }
        }
        .
    \end{align*}
    Similar to the proof of Theorem~\ref{thm:lb_low_cost_dinf}, due to the consistency requirement, and the same policy operating on $\nu$ where $a_k$ is sub-optimal, and on $\nu'$ where $a_k$ is optimal, the following holds,
    \begin{align*}
        \liminf_{ T \to \infty }
        \frac
        { \kl
        \prn*{ 
        \frac{ \Exp_{\nu} \brk*{ n_k (T) } }{ T },
        \frac{ \Exp_{\nu'} \brk*{ n_k (T) } }{ T }
        } }
        { \log T }
        &\geq
        1 - \gamma
        \quad \forall \,
        \gamma \in (0, 1]
        .
    \end{align*}
    Since the KL divergence $\KL \prn{ \nu_i, \nu_i' }$ does not depend on $T$, the initial lower bound can be combined with the bound on $\kl \prn{ \Exp_{\nu} \brk{Z}, \Exp_{\nu'} \brk{Z} }$ as, 
    \begin{align}
        \liminf_{T \to \infty}
        \frac{\Exp_{\nu} \brk*{ n_i (T) } }{\log T}
        &\geq
        \frac{1}{ \KL \prn*{ \nu_i, \nu_i' } }
        \,\, \forall \, 
        \nu_i' \in \mathcal{M} 
        \text{ such that }
        \Exp \brk{\nu_i'} = \mu_i' > \frac{\mu_{a^*}}{1 - \alpha},
    \end{align}
    Theorem~\ref{thm:lb_high_cost_dinf} is then just the tightest lower bound on the left hand side quantity.
\end{proof}
\begin{proof}[Proof of Theorem~\ref{thm:lb_high_cost}]
    Along the lines of the proof of Theorem~\ref{thm:lb_low_cost}, the proof of Theorem~\ref{thm:lb_high_cost} follows from imposing the distribution family $\mathcal{M}$ being Gaussian with unit variance onto Theorem~\ref{thm:lb_high_cost_dinf}. Under this restriction, $D_{\inf} \prn{ \nu_i, \frac{\mu_{a^*}}{1 - \alpha} } = \dfrac{\prn{ \mu_{a^*} - (1 - \alpha) \mu_i }^2}{2 (1 - \alpha)^2} \,\, \forall \, a_i \in \mathcal{A}^+$.
\end{proof}

\begin{proof}[Proof of Theorem~\ref{thm:lb_joint}]
    The joint lower bound is based on the insight that the optimality of $a_{a^*}$ may be upended by a cheap arm $a_k \in \mathcal{A}^-$ not only by enhancing the reward of said cheap arm to $(1 - \alpha) \mu^*$, but also by diminishing the reward of any arm $a_i \in \mathcal{A}$ whose reward is more than $\frac{\mu_i}{1 - \alpha}$. Let $\mathcal{A}^{k} \coloneqq \crl*{ a_i \in \mathcal{A} \mid (1 - \alpha) \mu_i > \mu_k }$ denote the set of all arms standing in the way of cheap arm $a_k$ being optimal. We construct $\nu'$ such that $\nu'_i = \nu_i \, \forall \, a_i \notin \mathcal{A}^k$, and $\nu'_i \in \mathcal{M}$ is such that $\Exp \brk{ \nu_i' } = \mu_i' \leq \frac{\mu_{k}}{1 - \alpha} \, \forall \, a_i \in \mathcal{A}^k$.

    We start by applying the fundamental inequality of Lemma~\ref{lemma:lb_fundamental_inequality} to the instance pair $\nu, \nu'$ with random variable $Z = \frac{\Exp \brk{n_k (T)}}{T}$. At the outset, we note that arm $a_k$ is sub-optimal for instance $\nu$, but optimal for $\nu'$ allowing us to leverage the consistency of the policy in a manner identical to the proofs of Theorems~\ref{thm:lb_low_cost_dinf}, and \ref{thm:lb_high_cost_dinf}. 
    \begin{align}
        \sum_{ a_i \in \mathcal{A}^{k} }
        \Exp \brk*{ n_k (T) } 
        \KL \prn*{ \nu_i, \nu_i' }
        &\geq
        \kl 
        \prn*{ 
        \frac{ \Exp_{\nu} \brk*{n_k (T)} }{ T }, 
        \frac{ \Exp_{\nu'} \brk*{n_k (T)} }{ T } 
        }
        \nonumber
        \\
        \implies
        \liminf_{T \to \infty}
        \frac
        {\displaystyle \sum_{ a_i \in \mathcal{A}^{k} }
        \Exp \brk*{ n_k (T) } 
        \KL \prn*{ \nu_i, \nu_i' }}
        {\log T}
        &\geq
        1
        \,\, \forall \, 
        \nu_i' \in \mathcal{M} 
        \text{ such that }
        \Exp \brk{\nu_i'} = \mu_i' \leq \frac{\mu_{k}}{1 - \alpha},
        \text{ for each } a_i \in \mathcal{A}^k
        \nonumber
        \\
        \implies
        \liminf_{T \to \infty}
        \frac
        {\displaystyle \sum_{ a_i \in \mathcal{A}^{k} }
        \Exp \brk*{ n_k (T) } 
        \tilde{D}_{\inf} \prn*{ \nu_i, \frac{\mu_{k}}{1 - \alpha} }
        }
        {\log T}
        &\geq
        1
        \quad
        \text{(using Definition~\ref{def:D_inf_definition} for $\tilde{D}_{\inf}$ to get the tightest lb)}
        .
        \nonumber
    \end{align}
    Once the restriction that the reward distribution family $\mathcal{M}$ is all Gaussian distributions with unit variance is imposed, we can leverage Lemma~\ref{lemma:d_inf_explicit_formula} to replace $\tilde{D}_{\inf} \prn*{ \nu_i, \frac{\mu_{k}}{1 - \alpha} }$ with $\dfrac{ \prn*{ (1 - \alpha)\mu_i - \mu_k }^2 }{2 (1 - \alpha)^2}$ to write,
    \begin{align*}
        \liminf_{T \to \infty}
        \frac
        {
        \displaystyle \sum_{ a_i \in \mathcal{A}^{k} }
        \prn*{ (1 - \alpha)\mu_i - \mu_k }^2
        \Exp \brk*{ n_k (T) } 
        }
        {\log T}
        &\geq
        2 (1 - \alpha)^2
        .
    \end{align*}
    The result above represents a collection of $\abs{ \mathcal{A}^- } = a^* - 1$ joint lower bounds, one for each cheap arm $a_k \in \mathcal{A}^{-}$. In Section~\ref{sec:theory_lower_bounds} we defined $a_{\dagger} = \arg \max_{a_i \in \mathcal{A}^- } \mu_i$ as the best reward cheap arm. This means that $\mathcal{A}^{\dagger} \subseteq \mathcal{A}^{k} \, \forall \, a_k \in \mathcal{A}^-$. Moreover, $ \prn*{ (1 - \alpha)\mu_i - \mu_{\dagger} }^2 \leq \prn*{ (1 - \alpha)\mu_i - \mu_k }^2 \, \forall \, a_k \in \mathcal{A}^-$. These two facts put together imply that the joint lower bound stated in Theorem~\ref{thm:lb_joint} is simply the tightest among these $\abs{\mathcal{A}^-}$ overlapping lower bounds.
\end{proof}

\begin{remark}[Joint lower bound is a generalization of \citep{juneja2025pairwise}]
    As stated in Section~\ref{sec:theory_lower_bounds} while we come to the same conclusion about the lower bound for cheap arms and the individual lower bound for expensive arms as \citep{juneja2025pairwise}, we highlight here how our Theorem~\ref{thm:lb_joint} is a strict generalization of the lower bound on $\Exp \brk*{n_{i^*} (T)}$ shown in \citep{juneja2025pairwise}. In \citep{juneja2025pairwise} the lower bound based on reduction of reward for $\Exp \brk*{n_{i^*} (T)}$ is,
    \begin{align}
        \liminf_{T \to \infty}
        \frac
        { \Exp \brk*{ n_{i^*} (T) } }
        {\log T}
        &\geq
        \max_{i < a^*}
        \frac
        {2 (1 - \alpha)^2}
        {
        \Delta_{Q, i}^2
        }
        \text{ if } 
        \mu_{\dagger} \geq (1 - \alpha) \mu_{(2)},
        \label{eq:iclr_lower_bound}
    \end{align}
    where $\mu_{(2)} \coloneqq \max_{i \neq i^*} \mu_i$ is the second largest reward. The bound in Equation~\ref{eq:iclr_lower_bound} is a special case of Theorem~\ref{thm:lb_joint}. When $\mu_{\dagger} \geq (1 - \alpha) \mu_{(2)}$, only the best reward arm $a_{i^*}$ stands in the way of cheap arm $a_{\dagger}$ being optimal and Theorem~\ref{thm:lb_joint} reduces to,
    \begin{align*}
        \liminf_{T \to \infty}
        \frac
        {
        \prn*{ (1 - \alpha)\mu^{*} - \mu_k }^2
        \Exp \brk*{ n_k (T) } 
        }
        {\log T}
        &\geq
        2 (1 - \alpha)^2
        \\
        \implies
        \liminf_{T \to \infty}
        \frac
        {
        \Exp \brk*{ n_k (T) } 
        }
        {\log T}
        &\geq
        \frac{ 2 (1 - \alpha)^2 }{ \Delta_{Q, \dagger}^2 }
        \\
        &=
        \max_{i < a^*}
        \frac{ 2 (1 - \alpha)^2 }{ \Delta_{Q, i}^2 },
    \end{align*}
    since $\mu_{\dagger} = \max_{a_i \in \mathcal{A}^-} \mu_i $.
\end{remark}

\subsubsection*{Example to demonstrate room for improvement in lower bound}

The following example illustrated how there is room for improvement in the lower bound on the expected number of samples of an expensive arm beyond the results proved in this paper. Using notation from the paper consider a four armed MAB-CS instance with arms $\mathcal{A} = \crl*{a_{\dagger}, a^*, a_k, i^*}$, and only best reward arm $i^*$ being capable of deeming $a_{\dagger}$ infeasible. The lower bound on the expected samples required from arbitrary $a_k$ can be tightened. Currently there is no dependence on the lower bound of $\Exp \brk*{n_k (T)}$ on the gap $\Delta_k = \mu^* - \mu_k$ even though in practice resolving $\Delta_k$ is necessary for directing sampling from $a_k$ towards $i^*$ to deem $a_{\dagger}$ infeasible.

\clearpage

\section{Analysis for COF}
\label{sec:cof_ub}

COF is carefully designed such that deviations from its intended execution have low probability. COF evaluates an arm $a_{\ell}$ as either feasible or infeasible by comparing it against all the arms more expensive than itself. First the feasibility criteria is checked while constructing the set of gating arms $\mathcal{G}_\ell$ (Line 3, Algorithm~\ref{algo:COF}), then the infeasibility criteria is checked (Lines 7-10, Algorithm~\ref{algo:COF}). If $a_{\ell}$ is found to be feasible then we sample it until the $T$ slots run out. In contrast, once $a_{\ell}$ is deemed infeasible, we move on to evaluating the next cheapest $a_{\ell + 1}$. If a decision about $a_{\ell}$ cannot yet be made with sufficient confidence, we further sample $a_{\ell}$ and arms in $\mathcal{G}_{\ell}$ (Lines 12-15, Algorithm~\ref{algo:COF}). 

Sampling of $a_\ell$ and arms $a_i \in \mathcal{G}_\ell$ is uniform barring two considerations. First, if the samples of $a_{\ell}$ lag behind those of arms in $\mathcal{G}_{\ell}$ then we exclusively sample $a_{\ell}$ until the disparity is amended. Second, if the $\UCB$ of an arm in $\mathcal{G}_{\ell}$ falls below the largest $\LCB$ among arms in $\mathcal{G}_{\ell}$ then its sampling is foregone in favor of arms that are better poised to gauge the feasibility of $a_{\ell}$. We refer to the former feature as {\em exclusive sampling} and the latter as the {\em BAI-filter} on sampling. In this section we build up to the bounds on cost and quality regret for COF (Algorithm~\ref{algo:COF}) that were stated in Theorem~\ref{thm:cof_regret_ub} by separately bounding the expected samples of cheap arms $\mathcal{A}^-$ and expensive arms $\mathcal{A}^{+}$. 

COF operates in an episodic fashion where the $\ell^{\text{th}}$ episode is responsible for evaluating arm $a_{\ell}$. The premise for analysis is that any arm $a_i$ can only be sampled in episodes $1, 2, \ldots, i$. In episodes $1, 2, \ldots, i - 1$ arm $a_i$ may be sampled as part of $\mathcal{G}_{\ell}$ and in episode $i$ as the candidate feasible arm. Analyzing $\Exp \brk*{n_i (T)}$ for a sub-optimal arm $i \neq a^*$ involves conditioning its number of samples on a carefully constructed event that secures normative sampling progression for the arm across episodes.    

\begin{definition}[The event $E_{i, \ell}$]
    $E_{i, \ell}$ denotes the event that the terminal sample for arm $i$ was sampled during episode $i$. To bound the samples of any sub-optimal arm $a_i \in \mathcal{A}$ we leverage the fact the collection of events $\crl*{ E_{i, \ell} }$ are mutually exclusive and exhaustive.
    \label{def:E_i_ell}
\end{definition}

\subsection{Bound Samples of Cheap Arms}
\label{sec:cheap_arms}
\begin{lemma}[Bound on number of samples of a cheap arm under COF]
    When the error tolerance is $\delta = T^{-2}$, the expected number of samples of arm $a_i \in \mathcal{A}^{-}$ by COF over horizon $T$ is upper bounded as,
    \begin{align*}
        \Exp 
        \brk{
            n_{i} (T)
        }
        &\leq
        \frac{16 \log T}{ \Delta_{Q, i}^2 } 
        +
        2
        .
    \end{align*}
    \label{lemma:cheap_arms_ub}    
\end{lemma}
\begin{proof}
We start by defining event $G_i$ conditioning on which shall ensure our intended normative sampling for arm $a_i$.
\begin{align*}
    G_{i}
    &=
    \crl*{ \abs*{ \hat{\mu}_i - \mu_i } 
            < \beta_i (t, \delta) \,\, \forall \, t \in [T]
        }
    \cap
    \crl*{ \abs*{ \hat{\mu}_{i^*} - \mu^* }
            < \beta_{i^*} (t, \delta) \,\, \forall \, t \in [T]
        }
    .
\end{align*}
Or equivalently, in purely set-theoretic notation,
\begin{align}
    G_{i}
    &=
    \crl*{ 
            \bigcap_{t=1}^{T}
            \abs*{ \hat{\mu}_i - \mu_i } 
            < \beta_i (t, \delta)
        }
    \cap
    \crl*{ 
            \bigcap_{t=1}^{T}
            \abs*{ \hat{\mu}_{i^*} - \mu^* }
            < \beta_{i^*} (t, \delta)
        }
    .
    \label{eq:G_i_set_definition}
\end{align}
Using Lemma~\ref{lemma:bound_by_conditioning} with $G_i$ we get,
\begin{align}
    \Exp \brk*{ n_i (T) }
    &\leq
    \Exp \brk*{ n_i (T) \mid G_i }
    +
    T \Pr{ G_i^c }
    \label{eq:initial_ub_cheap_arms}
    .
\end{align}
Henceforth we use notation $\Exp_{G} \brk*{ n_i(T) } \coloneqq \Exp \brk*{ n_i(T) \mid G_i} $ for the conditional expectation. From Definition~\ref{def:E_i_ell} we know that the collection of events $\crl*{ E_{i, \ell} }_{\ell = 1}^{i}$ is mutually exclusive and exhaustive. Hence, we can apply Lemma~\ref{lemma:expectation_less_than_max_lemma} to the conditional expectation $\Exp_{G} \brk*{ n_i(T) }$ with the events $\crl*{ E_{i, \ell} }_{\ell = 1}^{i}$ to obtain,
\begin{align}
    \Exp_G
    \brk*{ n_i (T) }
    &\leq
    \max_{\ell = 1, \ldots, i}
    \Exp_G
    \brk*{ n_i (T) \mid E_{i, \ell} }
    .
    \label{eq:cheap_arm_max_episode_bound}
\end{align}

When $G_{i}$ holds, due to its first clause related to the normativity of empirical mean $\mu_i$ we can upper bound $\UCB_i$ per, 
\begin{align}
    \UCB_i
    &=
    \hat{\mu}_i
    +
    \beta_i (t, \delta)
    \leq
    \mu_i
    +
    2 \beta_i (t, \delta)
    \,\,
    \forall
    t \in [T]
    .
    \label{eq:upper_bound_on_ucb_i}
\end{align}
Similarly, due to the second clause related to the normativity of the empirical mean of the best reward arm $\hat{\mu}_{i^*}$ we can lower bound $\LCB_{i^*}$ per,
\begin{align}
    \LCB_{i^*}
    &=
    \hat{\mu}_{i^*}
    -
    \beta_{i^*} (t, \delta)
    \leq
    \mu^{*}
    -
    2 \beta_{i^*} (t, \delta)
    \,\,
    \forall
    t \in [T]
    .
    \label{eq:lower_bound_on_lcb_i_star}
\end{align}
Now we separately consider the consequences of Equations~\ref{eq:upper_bound_on_ucb_i}, \ref{eq:lower_bound_on_lcb_i_star} during an arbitrary episode $\ell < i$ and during episode $i$ in bounding $\Exp_{G} \brk*{ n_i (T) \mid E_{i, \ell} }$ and $\Exp_{G} \brk*{ n_i (T) \mid E_{i, i} }$ respectively.

\subsubsection*{During Episode $\ell < i$}

To begin with both $a_i, a_{i^*} \in \mathcal{G}_{\ell}$, and due to $G_i$ the bounds on $\UCB_i$ and $\LCB_{i^*}$ stated in Equations~\ref{eq:upper_bound_on_ucb_i} and \ref{eq:lower_bound_on_lcb_i_star} respectively hold. If the upper bound on $\UCB_i$ falls below the lower bound on $\LCB_{i^*}$ there will be no further samples of $a_i$ due to the BAI-filter. This condition can be captured in terms of reward gap $\Delta_i = \mu^* - \mu_i$ based on the following simplification,
\begin{align*}
    \mu_i + 2 \beta_i (t, \delta)
    &<
    \mu^* - 2 \beta_{i^*} (t, \delta)
    \\
    \implies
    \beta_i (t, \delta)
    +
    \beta_{i^*} (t, \delta)
    &<
    \frac{\Delta_i}{2}
    \quad
    \text{ (definition of gap $\Delta_i$)}
    \\
    \implies
    \sqrt{\frac{1}{2 n_i (t)} \log \prn*{\frac{1}{\delta}} }
    +
    \sqrt{\frac{1}{2 n_{i^*} (t)} \log \prn*{\frac{1}{\delta}} }
    &<
    \frac{\Delta_i}{2}
    \quad
    \text{ (definition of confidence radius $\beta (t, \delta)$)}
    \\
    \sqrt{\frac{1}{2 n (t)} \log \prn*{\frac{1}{\delta}} }
    &<
    \frac{\Delta_i}{4}
    \quad
    \text{ (since $n_i (t) = n_{i^*} (t) = n(t)$ are matched until elimination)}
    .
\end{align*}

Let $\tau_{i, \ell}$ represent both $n_i, n_{i^*}$ beyond which there is a separation between the upper bound on $\UCB_i$ and the lower bound on $\LCB_{i^*}$. Then,
\begin{align*}
    \tau_{i, \ell}
    &=
    \min \crl*{
        n
        \mid 
        \sqrt{\frac{1}{2 n} \log \prn*{\frac{1}{\delta}} }
        <
        \frac{\Delta_i}{4}
    }
    \\
    &= \frac{8 \log \prn{1 / \delta} }{\Delta_i^2}
    \quad
    \text{ (rearranging terms and leveraging monotonicity in $n$)}
    \\
    &= \frac{16 \log T }{\Delta_i^2}
    \quad
    \text{ (since $\delta = T^{-2}$ in Lemma~\ref{lemma:cheap_arms_ub})}
    .
\end{align*}
We conclude that $\Exp_{G} \brk*{ n_i (T) \mid E_{i, \ell} } \leq \tau_{i, \ell} = \dfrac{16 \log T }{\Delta_i^2}$.

\subsubsection*{During Episode $i$}

During episode $i$, arm $a_i$ is the candidate arm and $a_{i^*}$ is in $\mathcal{G}_{\ell}$. Unlike in episodes $\ell < i$, sampling of $a_i$ is not subject to the BAI-filter. To bound $\Exp_{G} \brk*{ n_i (T) \mid E_{i, i} }$ we again leverage the bounds on $\UCB_i$ and $\LCB_{i^*}$ stated in Equations~\ref{eq:upper_bound_on_ucb_i} and \ref{eq:lower_bound_on_lcb_i_star} respectively. However, instead of a lower bound $\LCB_{i^*}$ we use a lower bound on $(1 - \alpha) \LCB_{i^*}$. 
\begin{align*}
    (1 - \alpha) 
    \LCB_{i^*} 
    &=
    (1 - \alpha)
    \prn*{ \hat{\mu}_{i^*} - \beta_{i^*} (t, \delta) }
    \geq
    (1 - \alpha)
    \prn*{
        \mu^*
        -
        2 \beta_{i^*}
        (t, \delta)
    }
    \,\,
    \forall
    t \in [T]
    .
    \label{eq:subsidized_lower_bound_on_lcb_i_star}
\end{align*}

The key insight is that arm $a_i \in \mathcal{A}^-$ shall necessarily be deemed infeasible by COF by when $\epsilon_{i^*, i} < \delta$, making the overall product, $\prod_{a_k \in \mathcal{A}} \epsilon_{k, i} < \delta$ necessarily. Analogous to $\tau_{i, \ell}$, we find the smallest $\tau_i$ for which $\epsilon_{i^*, i} < \delta$ holds under $G_i$. 

First we demonstrate through a rearrangement of definitions that $\UCB_i < (1 - \alpha) \LCB_{i^*} \implies \epsilon_{i^*, i} < \delta$. 
\begin{align*}
    \UCB_i
    &<
    (1 - \alpha)
    \LCB_{i^*}
    =
    (1 - \alpha)
    \prn*{
        \hat{\mu}_{i^*}
        -
        \beta_{i^*}
        \prn*{ t, \delta }
    }
    \\
    &\implies
    \beta_{i^*} (t, \delta)
    <
    \frac{\UCB_i}{ (1 - \alpha) }
    - 
    \hat{\mu}_{i^*}
    .
\end{align*}
But we know that $\epsilon_{i^*, i}$ is such that,
\begin{align}
    \beta_{i^*}
    (t, \epsilon_{i^*, i})
    &=
    \dfrac{\UCB_i}{(1 - \alpha)}
    -
    \hat{\mu}_{i^*}
    .
    \\
    \implies
    \beta_{i^*}
    (t, \epsilon_{i^*, i})
    &>
    \beta_{i^*}
    (t, \delta)
    .
    \\
    \implies
    \epsilon_{i^*, i} 
    &< 
    \delta
    .
\end{align}
Since $\epsilon_{k, i} \leq 1 \,\, \forall \, a_k \in \mathcal{A}$, $\prod_{a_k \in \mathcal{A}} \epsilon_{k, i} < \delta $. Therefore, arm $a_i$ is deemed infeasible by the time we accrue samples equal to those which can guarantee a separation between $\UCB_{i}$ and $(1 - \alpha) \LCB_{i^*}$. From bounds in~\ref{eq:upper_bound_on_ucb_i} and \ref{eq:subsidized_lower_bound_on_lcb_i_star} this condition is,
\begin{align*}
    \mu_i
    +
    2 \beta_{i} (t, \delta)
    &<
    (1 - \alpha)
    \prn*{
        \mu^*
        -
        2 \beta_{i^*}
        (t, \delta)
    }
    \\
    \implies
    2 + 2 (1 - \alpha) \beta_i (t, \delta)
    &<
    \Delta_{Q, i}
    .
\end{align*}
The coalescing of the $\beta$ terms follows from $n_i (t) \leq n_{i^*} (t) \implies \beta_i (t, \delta) \geq \beta_{i^*} (t, \delta)$ during episode $a_i$. Therefore, the desired bound $\tau_i$ on $\Exp_G \brk*{ n_i (T) \mid E_{i, i} }$ is,
\begin{align*}
    \tau_i
    &=
    \min
    \crl*{
        n_i (t)
        \mid
        2 + 2 (1 - \alpha) \beta_i (t, \delta)
        <
        \Delta_{Q, i}
    }
    \\
    &\leq
    \min
    \crl*{
        n_i (t)
        \mid
        \beta_i (t, \delta)
        < \frac{\Delta_{Q, i}}{4}
    }
    \quad
    \text{ (since $\alpha \in [0, 1)$)}
    \\
    &=
    \min
    \crl*{
        n_i (t)
        \mid
        \sqrt{ \frac{1}{2 n_i (t)} \log \prn*{\frac{1}{\delta}} }
        <
        \frac{\Delta_{Q, i}}{4}
    }
    \quad
    \text{ (definition of $\beta_i (t, \delta)$)}
    \\
    &=
    \frac{8 \log \prn{1 / \delta} }{\Delta^2_{Q, i} }
    \\
    &=
    \frac{16 \log T}{\Delta^2_{Q, i}}
    \quad
    \text{ (since $\delta = T^{-2}$ in Lemma~\ref{lemma:cheap_arms_ub})}
    .
\end{align*}    
For any cheap arm $a_i \in \mathcal{A}^{-}$, since $\mu_i < (1 - \alpha) \mu^*$, the quality gap $\Delta_{Q, i}^2 = (1 - \alpha) \mu^*$ is necessarily smaller than the reward gap $\Delta_i = \mu^* - \mu_i$. Therefore the bound in \ref{eq:cheap_arm_max_episode_bound} can be resolved as,
\begin{align*}
    \Exp
    \brk*{ n_i (T) \mid G}
    &\leq
    \max
    \crl*{\tau_{i, \ell}, \tau_i}
    =
    \frac{16 \log T}{\Delta_{Q, i}^2}
    .
\end{align*}
To complete the proof of Lemma~\ref{lemma:cheap_arms_ub} we must also bound the probability of $G_i$ not holding $\Pr{ G_i^c }$. Using De-morgan's law and the definition of $G_i$ in Equation~\ref{eq:G_i_set_definition},
\begin{align*}
    \Pr{ G_i^c }
    &=
    \Pr{ 
    \crl*{ 
        \bigcup_{t=1}^{T}
        \abs*{ \hat{\mu}_i - \mu_i } 
        \geq \beta_i (t, \delta)
    }
    \cup
    \crl*{ 
        \bigcup_{t=1}^{T}
        \abs*{ \hat{\mu}_{i^*} - \mu^* }
        \geq \beta_{i^*} (t, \delta)
    }
    }
    \\
    &\leq
    \sum_{t=1}^{T}
    \Pr{ 
    \abs*{ \hat{\mu}_i - \mu_i } 
    \geq \beta_i (t, \delta)
    }
    +
    \sum_{t=1}^T
    \Pr{
    \abs*{ \hat{\mu}_{i^*} - \mu^* }
    \geq \beta_{i^*} (t, \delta)
    }
    \quad
    \text{ (union bound)}
    \\
    &\leq
    2 T \delta
    \quad
    \text{ (using Lemma~\ref{lemma:ucb_lcb_incorrect})}
    .
\end{align*}
Plugging the bound on $\Pr{G_i^c}$ in Equation~\ref{eq:initial_ub_cheap_arms} and substituting in $\delta = T^{-2}$ we get the bound on $\Exp \brk*{ n_i (T) }$ stated in Lemma~\ref{lemma:cheap_arms_ub}.
\end{proof}

\newpage

\subsection{Bound Samples of Expensive Arms}
\label{sec:expensive_arms_analysis}
The samples of an arbitrary expensive arm $a_i \in \mathcal{A}^+$ may only be accrued during episodes $\ell \in \crl{ 1, 2, \ldots, a^*, \ldots, i }$. Under its normative progression, COF deems arm $a^*$ feasible. Therefore the contribution to $n_i (T)$ from any episode beyond $a^*$ is statistically insignificant. We establish this formally through Lemma~\ref{lemma:normative_episode_a_star}.
\begin{lemma}[Normative progression of COF ends in episode $a^*$]
    Let $E_{i, \ell}$ be the event that the terminal sample of arm $a_i$ is drawn during episode $\ell$. Then the expected number of samples for an expensive arm $a_i \in \mathcal{A}^{+}$ under COF over horizon $T$ can be upper bounded as,
    \begin{align*}
        \Exp \brk*{ n_i (T) }
        &\leq
        \max
        \crl*{
        \max_{\ell < a^*}
        \crl*{
        \Exp 
        \brk*{ n_i (T) \mid E_{i, \ell} }
        },
        \,\,
        \Exp 
        \brk*{ 
        n_i (T) 
        \mid
        E_{i, a^*},
        G_{a^*}
        }
        +
        2 T^2 \delta
        },
    \end{align*}
    where $G_{a^*}$ is the event,
    \begin{align}
        G_{a^*}
        &=
        \crl*{
        \abs*{ \hat{\mu}_{a^*} - \mu_{a^*} } < \beta_{a^*} (t, \delta)
        \,\, \forall \, t \in [T]
        }
        \cap
        \crl*{ 
        \prod_{a_k \in \mathcal{A}} \epsilon_{k, a^*} (t) 
        > 
        \delta
        \,\, \forall \,
        t \in [T]
        }
        .
        \label{eq:G_a_star_definition}
    \end{align}
    \label{lemma:normative_episode_a_star}
\end{lemma}
\begin{proof}
Let random variable $Z$ denote the terminal episode in which arm $a_i$ was sampled. Then $Z$ can take values in $\crl*{1, 2, \ldots, a^*, \ldots, i}$ and the events $\crl{Z < a^*}, \crl{Z \geq a^*}$ are mutually exclusive and exhaustive. Applying Lemma~\ref{lemma:expectation_less_than_max_lemma} to the expected samples $\Exp \brk*{ n_i (T) }$ we have,  
\begin{align*}
    \Exp \brk*{ n_i (T) }
    &\leq
    \max 
    \crl*{
    \Exp \brk*{ n_i (T) \mid \crl{Z < a^*} },
    \Exp \brk*{ n_i (T) \mid \crl{Z \geq a^*} }
    }
    .
\end{align*}
Just like in the proof of Lemma~\ref{lemma:cheap_arms_ub} let $E_{i, \ell}$ denote the event that the terminal sample of arm $a_i$ was drawn during episode $\ell < a^*$. Then again using Lemma~\ref{lemma:expectation_less_than_max_lemma},
\begin{align*}
    \Exp \brk*{ n_i (T) \mid \crl{Z < a^*} }
    &\leq
    \max_{\ell < a^*}
    \crl*{
    \Exp 
    \brk*{ n_i (T) \mid E_{i, \ell} }
    }
    .
\end{align*}
For the $Z \geq a^*$ case we introduce conditioning on the event $G_{a^*}$ (Equation \ref{eq:G_a_star_definition}) that ensures the normative progression of COF during episode $a^*$. 
\begin{align*}
    G_{a^*}
    &=
    \bigcap_{t = 1}^{T}
    \crl*{
    \abs*{ \hat{\mu}_{a^*} - \mu_{a^*} } < \beta_{a^*} (t, \delta)
    }
    \cap
    \bigcap_{t = 1}^{T}
    \crl*{ 
    \prod_{a_k \in \mathcal{A}} \epsilon_{k, a^*} (t) 
    > 
    \delta
    }
    ,
\end{align*}
where the latter clause is based on the infeasibility condition of COF. Now we can bound $\Exp \brk*{ n_i (T) \mid \crl{Z \geq a^*} }$ using Lemmas~\ref{lemma:bound_by_conditioning} and the event $G_{a^*}$ as,
\begin{align}
    \Exp \brk*{ n_i (T) \mid \crl{Z \geq a^*} }
    &\leq
    \Exp \brk*{ n_i (T) \mid \crl{Z \geq a^*}, G_{a^*} }
    +
    T 
    \Pr{ G_{a^*}^c \mid \crl{Z \geq a^*} }
    \nonumber
    \\
    &=
    \Exp \brk*{ n_i (T) \mid \crl{Z \geq a^*}, G_{a^*} }
    +
    T 
    \Pr{ G_{a^*}^c }
    \quad
    \text{ ($G_{a^*}$ is not specified by episode)}
    \label{eq:initial_geq_a_star_bound}
    .
\end{align}
The bound on $\Pr{G_{a^*}^c}$ follows as,
\begin{align*}
    \Pr{G_{a^*}^c}
    &=
    \Pr{
    \bigcup_{t = 1}^T
    \crl*{
    \abs*{ \hat{\mu}_{a^*} - \mu_{a^*} } \geq \beta_{a^*} (t, \delta)
    }
    \cup
    \bigcup_{t=1}^T
    \crl*{ 
    \prod_{a_k \in \mathcal{A}} \epsilon_{k, a^*} (t) 
    \leq
    \delta
    }
    }
    \quad
    \text{ (De-morgan's rule on definition)}
    \\
    &\leq
    \sum_{t = 1}^{T}
    \Pr{ 
    \abs*{ \hat{\mu}_{a^*} - \mu_{a^*} } \geq \beta_{a^*} (t, \delta)
    }
    +
    \sum_{t = 1}^T
    \Pr{
    \prod_{a_k \in \mathcal{A}} \epsilon_{k, a^*} (t) 
    \leq
    \delta
    }
    \quad
    \text{ (union bound)}
    \\
    &\leq
    2 T \delta
    \text{ (Lemma~\ref{lemma:ucb_lcb_incorrect} and by construction for $\epsilon_{k, a^*}$)}
    .
\end{align*}
Plugging the bound on $\Pr{ G_i^c }$ into \ref{eq:initial_geq_a_star_bound},
\begin{align}
    \Exp
    \brk*{
    n_i (T)
    \mid
    \crl*{Z \geq a^*}
    }
    &\leq
    \Exp 
    \brk*{ 
    n_i (T) 
    \mid
    \crl*{Z \geq a^*},
    G_{a^*}
    }
    +
    2 T^2 \delta
    \nonumber
    \\
    &=
    \Exp 
    \brk*{ 
    n_i (T) 
    \mid
    E_{i, a^*},
    G_{a^*}
    }
    +
    2 T^2 \delta
    \label{eq:ep_a_star_final_equivalence}
    .
\end{align}
Where the final equivalence follows from the insight that the joint event $\crl*{Z \geq a^*}, G_{a^*}$ is equivalent to $E_{i, a^*}, G_{a^*}$ since $G_{a^*}$ precludes progression to any episode $Z > a^*$. Plugging in the bound \ref{eq:ep_a_star_final_equivalence} into \ref{eq:initial_geq_a_star_bound} and combining the result with the simplification for episodes $\ell < a^*$ we get the bound stated in Lemma~\ref{lemma:normative_episode_a_star}.
\end{proof}

In addition to Lemma~\ref{lemma:normative_episode_a_star}, we need some additional insights and definitions to analyze samples $n_i$ during episodes $\ell < a^*$. The BAI-filter of COF (Algorithm~\ref{algo:COF}, Line 15) ensures that any arm $a_i \in \mathcal{G}_{\ell}$ is not further sampled once we determine that the arm is not the highest reward arm $i^*$ with sufficient confidence. As discussed in Section~\ref{sec:cof}, the {\em BAI-filter} helps reduce regret because in the absence of the knowledge of cost gaps (costs are known but cost gaps $\Delta_{C, i}$ are unknown), the most regret efficient way of determining the infeasibility or feasibility of an arm are samples of arm $i^*$. Therefore once an arm is known to not be $i^*$, our limited sample budget is better off being allocated to other arms in $\mathcal{G}_\ell$. 

In any particular episode $\ell < a^*$ the normative outcome is for $a_\ell$ to be deemed infeasible. Some arms $a_i, i > \ell$ will have their sampling curtailed by the BAI-filter, and others will continue to be sampled till $a_\ell$ is deemed infeasible. Prior to bounding the samples of an expensive arm during episode $\ell$, we introduce some definitions ($\Delta_{k, \ell}, \mathcal{A}^{\ell}$) and an intermediate result (Lemma~\ref{lemma:optimization_results_for_infeasible_end}) to distinguish between the two possible fates of an expensive arms samples. 
\begin{definition}[The gap $\Delta_{i, \ell}$ for MAB-CS]
    For an MAB-CS instance $\nu$ with bandit arms $a_i \in \mathcal{A}$ and subsidy factor $\alpha$. If $\Exp \brk{\nu_i} = \mu_i$, then we define $\Delta_{i, \ell} \coloneqq (1 - \alpha) \mu_i - \mu_\ell$. Intuitively the $\Delta_{i, \ell}$ captures the ease with which the feasibility of arm $a_\ell$ relative to arm $a_i$ may be resolved.
    \label{def:gap_i_ell}
\end{definition}

\begin{definition}[The set $\mathcal{A}^{\ell}$]
    The set $\mathcal{A}^{\ell}$ is the collection of all arms that have sufficient reward for deeming cheap arm $a_\ell \in \mathcal{A}^-$ infeasible. Mathematically,
    \begin{align*}
        \mathcal{A}^{\ell}
        \coloneqq
        \crl*{
        a_i \in \mathcal{A}
        \mid
        \Delta_{i, \ell} > 0
        },
        \quad
        A^{\ell} \coloneqq \abs*{\mathcal{A}^{\ell}}.
    \end{align*}
    Where $\Delta_{i, \ell}$ is from Definition~\ref{def:gap_i_ell}.
    \label{def:A_ell_definition}
\end{definition}
\begin{remark}[$\mathcal{A}^{\dagger} \subseteq \mathcal{A}^{\ell} \,\, \forall \, \ell < a^*$]
Let $a_{\dagger} \coloneqq \arg \max_{\crl*{i \mid a_i \in \mathcal{A}^{-}}} \mu_i$ denote the highest reward cheap arm. All arms in $\mathcal{A}^{\dagger}$ will also have sufficient reward to deem $a_{\ell}$ with $\mu_{\ell} \leq \mu_{\dagger}$ infeasible and therefore will also be in $\mathcal{A}^{\ell} \,\, \forall \,\, \ell < a^*$.   
\label{remark:the_set_A_dagger}
\end{remark}
Next we develop an intermediate result relating the sum of squared gaps to the number of sampling rounds. We shall later leverage Lemma~\ref{lemma:optimization_results_for_infeasible_end} to bound the sampling rounds needed to deem a cheap arm $a_{\ell} \in \mathcal{A}^{-}$ infeasible.
\begin{lemma}[Intermediate result relating samples and sum-squared gaps]
Let $\delta$ be the error tolerance, and let $\beta \coloneqq \sqrt{\frac{1}{2 n} \log \prn*{ \frac{1}{\delta} } }, \,\, n \in \nats, \, \delta \in \reals^{+}$ denote the confidence radius. Then for any $A$-sized collection $\crl*{ \tilde{\Delta}_k }_{k = 1}^{A}$, $\tilde{\Delta}_{k} \in \reals^{+}$,
\begin{align*}
    \min
    \crl*{
    n
    \,\,
    \scalebox{2.0}{$\mid$}
    \,
    \sum_{k=1}^{A}
    \prn*{ \tilde{\Delta}_{k}  - 3 \beta }^2
    > 
    \beta^2
    }
    &\leq
    \frac{\prn*{3 \sqrt{A} + 1}^2}{2}
    \frac{\log \prn*{1 / \delta} }{ \tilde{\Delta}^2 }
    .
\end{align*}
Where $\tilde{\Delta}$ is the root of sum of squares of $\crl*{ \tilde{\Delta}_k }_{k = 1}^{A},
\quad \tilde{\Delta} \coloneqq \sqrt{\sum_{k = 1}^{A} \tilde{\Delta}_k^2}$.
\label{lemma:optimization_results_for_infeasible_end}
\end{lemma}
\begin{proof}
We shall prove the result by upper bounding the minimum number of samples under the constraint specified in Lemma~\ref{lemma:optimization_results_for_infeasible_end} by the minimum number of samples that satisfy a strictly stronger constraint.
\begin{align*}
    \min
    \crl*{
    n
    \,\,
    \scalebox{2.0}{$\mid$}
    \,
    \sum_{k=1}^{A}
    \prn*{ \tilde{\Delta}_{k}  - 3 \beta }^2
    > 
    \beta^2
    }
    &=
    \min
    \crl*{
    n
    \,\,
    \scalebox{2.0}{$\mid$}
    \,
    \sum_{k=1}^{A}
    \tilde{\Delta}_k^2
    +
    9 A \beta^2
    - 
    6 \beta
    \sum_{k = 1}^{A}
    \tilde{\Delta}_k
    > 
    \beta^2
    }
    \\
    &\leq
    \min
    \crl*{
    n
    \,\,
    \scalebox{2.0}{$\mid$}
    \,
    \tilde{\Delta}^2
    +
    9 A \beta^2
    - 
    6 \sqrt{A} 
    \tilde{\Delta} 
    \beta
    > 
    \beta^2
    }
    \quad
    \text{ ($\tilde{\Delta}$ definition, and Lemma~\ref{lemma:rms_geq_am})}
    \\
    &=
    \min
    \crl*{
    n
    \,\,
    \scalebox{2.0}{$\mid$}
    \,
    (9A - 1)
    \beta^2
    - 
    6 \sqrt{A} 
    \tilde{\Delta} 
    \beta
    +
    \tilde{\Delta}^2
    > 
    0
    }
    .
\end{align*}
Consider the quadratic expression $g (\beta) = (9A - 1) \beta^2 - 6 \sqrt{A} \tilde{\Delta} \beta + \tilde{\Delta}^2 $. For some $c_1, c_2 \in \reals^{+}$ we create another related quadratic $h( \beta )$ as the factorization,
\begin{align*}
    h (\beta)
    &=
    \prn*{
    (9A - 1)
    \beta
    -
    \frac{ \tilde{\Delta} }{c_1}
    }
    \cdot
    \prn*{
    \beta 
    -
    \frac{ \tilde{\Delta} }{c_2}
    }
    =
    (9A - 1)
    \beta^2
    -
    \tilde{\Delta}
    \prn*{
    \frac{9A - 1}{c_2}
    +
    \frac{1}{c_1}
    }
    \beta
    +
    \frac{\tilde{\Delta}^2}{c_1 c_2}
    .
\end{align*}
If we can find $c_1, c_2$ such that $h (\beta) \leq g(\beta) \,\, \forall \, \beta$, then $h (\beta) > 0$ will imply $g (\beta) > 0$ and $\min \crl*{n \mid g(\beta) > 0} \leq \min \crl*{n \mid h(\beta) > 0}$. The solution region on $\beta$ for the constraint $h(\beta) > 0$ is $\beta < \frac{\tilde{\Delta}}{ (9A - 1) c_1 }$ or $\beta > \frac{ \tilde{\Delta} }{c_2}$. The latter solution here provides a lower limit on $\beta$ and for this reason is not meaningful for sample bounds in a bandit algorithm.

For $h (\beta) \leq g(\beta)$ to hold we need to pick $c_1, c_2 > 0$ such that,
\begin{align*}
    \frac{1}{(9A - 1) c_1}
    &\leq
    \frac{1}{c_2}
    \\
    \frac{9A - 1}{c_2}
    +
    \frac{1}{c_1}
    &\geq
    6 \sqrt{A},
    \quad
    A \in \nats
    \\
    \frac{1}{c_1 c_2}
    \leq
    1
    .
\end{align*}
We can achieve this by plugging in the constraints on $c_1, c_2$ into a symbolic solver to achieve a valid factorization that satisfies all the constraints needed for $h (\beta) \leq g(\beta)$. Using a symbolic solver, one such factorization is, 
\begin{align*}
    c_1 
    &=
    \frac
    {3 \sqrt{A} + 1 }
    {9A - 1}
    \\
    .
    c_2
    &=
    \frac{9A - 1}{3 \sqrt{A} + 1}
    .
\end{align*}
Continuing to upper bound $\min \crl*{n \mid g(\beta) > 0}$ by plugging in the expression for $c_1$,
\begin{align*}
    \min
    \crl*{
    n
    \,\,
    \scalebox{2.0}{$\mid$}
    \,
    \sum_{k=1}^{A}
    \prn*{ \tilde{\Delta}_{k}  - 3 \beta }^2
    > 
    \beta^2
    }
    &\leq
    \min
    \crl*{
    n
    \,\,
    \scalebox{2.0}{$\mid$}
    \,
    \beta
    =
    \sqrt{
    \frac{1}{2 n}
    \log
    \prn*{\frac{1}{\delta}}
    }
    <
    \frac{\tilde{\Delta}}{3 \sqrt{A} + 1}
    }
    \\
    &=
    \frac{\prn*{3 \sqrt{A} + 1}^2}{2}
    \frac{\log \prn*{1 / \delta} }{ \tilde{\Delta}^2 }
    \quad
    \text{(after rearranging constraint to be in $n$)}
    .
\end{align*}
\end{proof}
We now have the requisite tools needed to bound the samples of an expensive arm that end in an episode $\ell < a^*$.
\begin{lemma}[Bound on number of samples of an expensive arm in episode $\ell < a^*$]
    The expected number of samples of expensive arm $a_i \in \mathcal{A}^{+}$, conditioned on its final sample being drawn in episode $\ell$, over horizon $T$ is upper bounded as,
    \begin{align*}
        \Exp
        \brk*{
        n_i (T)
        \mid
        E_{i, \ell}
        }
        &\leq
        \min
        \crl*{
        \frac{8 \log \prn*{ 1 / \delta }}{\Delta_i^2},
        \tau_{\ell} (\delta)
        }
        + K T^2 \delta
        .
    \end{align*}
    Inside of the $\min$, the first operand represents the case of the sampling of $a_i$ being curtailed by the BAI-filter and the second case represents arm $a_i$ being sampled to the end of episode $\ell$. $\displaystyle \tau_{\ell} (\delta) \coloneqq \frac{ \prn*{3\sqrt{A} + 1}^2 }{2} \frac{\log \prn*{1 / \delta}}{\sum_{i=1}^{A} \Delta_{\phi(i), \ell}^2 }$, where $\phi(i)$ denotes the index of the $i^{\text{th}}$ highest reward arm in $\mathcal{A}^{\ell}$, and A represents the number of highest-reward arms $\crl*{a_{\phi(1)}, a_{\phi(2)}, \ldots, a_{\phi(A)} } \subseteq \mathcal{A}^{\ell}$ that collectively deem $a_{\ell}$ infeasible.
    \label{lemma:expensive_in_episode_ell}
\end{lemma}
\begin{proof}
We define an event $G_{i, \ell}$ conditioning on which secures normative progression for sampling of arm $a_i$ in episode $\ell$. 
\begin{align*}
    G_{i, \ell}
    &=
    \underbrace{
    \bigcap_{t = 1}^{T}
    \crl*{ 
    \abs*{ \hat{\mu}_i - \mu_i } 
    < \beta_i (t, \delta) 
    }
    }_{\text{Clause A}}
    \cap
    \underbrace{
    \bigcap_{t = 1}^{T}
    \crl*{ 
    \abs*{ \hat{\mu}_{\ell} - \mu_{\ell} }
    < \beta_{\ell} (t, \delta)
    }
    }_{\text{Clause B}}
    \cap
    \underbrace{
    \bigcap_{t = 1}^{T}
    \crl*{
    \bigcap_{a_k \in \mathcal{A}^{\ell}}
    \crl*{
        \abs{ \hat{\mu}_k  - \mu_k }
        <
        \beta_k ( t, \delta )
    }    
    }
    }_{\text{Clause C}}
    .
\end{align*}
$G_{i, \ell}$ consists of three clauses. Clause A secures the normative progression of the sampling of arm $i$ with respect to the BAI-filter. Clauses B and C in conjunction ensure the normative outcome of arm $a_{\ell}$ being deemed infeasible. To be precise, Clause A implies,
\begin{align}
    \UCB_i = \hat{\mu}_i + \beta_i \prn*{t, \delta} \leq \mu_i + 2 \beta_i \prn*{t, \delta} \,\, \forall \, t \in [T].
    \label{eq:clause_A}
\end{align}
Clause B implies,
\begin{align}
    \UCB_{\ell} = \hat{\mu}_{\ell} + \beta_{\ell} \prn*{t, \delta} \leq \mu_{\ell} + 2 \beta_{\ell} \prn*{t, \delta} \,\, \forall \, t \in [T].
    \label{eq:clause_B}
\end{align}
And clause C implies,
\begin{align}
    \LCB_{k} = \hat{\mu}_k - \beta_k \prn*{t, \delta} \geq \mu_k - 2 \beta_k \prn*{t, \delta} \,\, \forall \, a_k \in \mathcal{A}^{\ell}, \, \forall \, t \in [T].
    \label{eq:clause_C_part1}
\end{align}
Since $i^* \in \mathcal{A}^{\ell} \, \forall \, \ell < a^*$, clause C includes,
\begin{align}
    \LCB_{i^*} = \hat{\mu}_{i^*} - \beta_{i^*} (t, \delta) \geq \mu^* - 2 \beta_{i^*} (t, \delta) \,\, \forall \, t \in [T].
    \label{eq:clause_C_part2}
\end{align}

Define the shorthand $\Exp_{\ell} \brk*{ n_i (T) } \coloneqq \Exp \brk*{ n_i (T) \mid E_{i, \ell} }$. Then we show that conditioned on $G_{i, \ell}$, the sampling of arm $a_i$ during episode $\ell$ will be curtailed either by the BAI-filter or by the arm $a_\ell$ being deemed infeasible. An overall upper bound on $\Exp_{\ell} \brk*{ n_i (T) \mid G_{i, \ell} }$ thus will be from the earlier of the two possible fates. Applying Lemma~\ref{lemma:bound_by_conditioning} to $\Exp_{\ell} \brk*{ n_i (T) }$ with $G_{i, \ell}$,
\begin{align}
    \Exp_{\ell} \brk*{ n_i (T) }
    &\leq
    \Exp_{\ell} \brk*{ n_i (T) \mid G_{i, \ell} }
    +
    T \Pr{ G_{i, \ell}^c \mid E_{i, \ell} }
    \nonumber
    \\
    &=
    \Exp_{\ell} \brk*{ n_i (T) \mid G_{i, \ell} }
    +
    T \Pr{ G_{i, \ell}^c }
    \quad
    \text{ (since $G_{i, \ell}$ is defined $\forall \, t \in [T]$ agnostic to episode)}
    \label{eq:initial_sample_bound_expensive_i}
    .
\end{align}
Conditioned on $G_{i, \ell}$, first we determine a limit $\tau^{\textrm{BAI}}_{i}$ on the samples of arm $a_i$ due to the BAI-filter. Once $\UCB_i < \LCB_{i^*}$, arm $i$ will accrue no further samples during episode $\ell$. Therefore,
\begin{align}
    \tau^{\textrm{BAI}}_{i} 
    &=
    \min
    \crl*{
    n_i
    \mid
    \UCB_i 
    < 
    \LCB_{i^*}
    }
    \nonumber
    \\
    &\leq
    \min
    \crl*{
    n_i
    \mid
    \mu_i + 2 \beta_{i} (t, \delta)
    < 
    \mu^* - 2 \beta_{i^*} (t, \delta)
    }
    \nonumber
    \\
    &\leq
    \min
    \crl*{
    n_i
    \mid
    4 \beta_{i} (t, \delta)
    < 
    \mu^* - \mu_i
    }
    \quad
    \text{ (since $n_i (t) \leq n_{i^*} (t) \implies \beta_i (t, \delta) > \beta_{i^*} (t, \delta)$)}
    \nonumber
    \\
    &=
    \min
    \crl*{
    n_i
    \mid
    \sqrt{ \frac{1}{2 n_i} \log \prn*{ 1 / \delta }  }
    < 
    \frac{\Delta_i}{4}
    }
    \quad
    \text{ (using definition of $\beta_i (t, \delta)$)}
    \nonumber
    \\
    &=
    \frac{8 \log \prn*{1 / \delta} }{ \Delta_i^2 }
    \quad
    \text{ (rearranging the constraint to be in $n_i$)}
    \label{eq:tau_bai_final}
    .
\end{align}
Next we determine a limit on the number of samples of $a_i$ due to episode $\ell$ terminating. From Algorithm~\ref{algo:COF}, arm $a_{\ell}$ is deemed infeasible once $\prod_{ a_k \in \mathcal{A} } \epsilon_{k, \ell} (t) < \delta$ (infeasibility criteria). We bound the number of samples needed for the infeasibility criteria to hold by upper bounding the product $\prod_{a_k \in \mathcal{A}} \epsilon_{k, \ell}$, and determining the sample count that is sufficient for the upper bound to be less than $\delta$. 

The formula for $\epsilon_{k, \ell} (t)$\footnote{Going forward we drop explicit time parameterization $(t)$ for $\epsilon, \hat{\mu}, \UCB, \text{ and } \LCB$ unless necessary} is given by,
\begin{align} 
    \epsilon_{k, \ell} 
    &= 
    \begin{cases} 
        \exp
        \prn*{
        \!
        -2 n_k
        \prn*{ \hat{\mu}_k - \dfrac{\UCB_\ell}{(1 - \alpha)} }^2 
        },
        \text{ if } \hat{\mu}_k > \frac{\UCB_\ell}{(1 - \alpha)}. \\ 
        1, \quad \quad \text{otherwise}. 
    \end{cases}
    \label{eq:eps_def}
\end{align}
Let $\hat{\Delta}_{k, \ell} \coloneqq (1 - \alpha) \hat{\mu}_k - \UCB_{\ell} $, and $\hat{\Delta}_{k, \ell}^{+} = \max \crl{ 0, \hat{\Delta}_{k, \ell} }$. Then, 
\begin{align*}
    \prod_{a_k \in \mathcal{A}}
    \epsilon_{k, \ell}
    &=
    \exp
    \prn*{
    \frac{-2}{(1 - \alpha)}
    \sum_{a_k \in \mathcal{A}}
    n_k 
    \prn*{ \hat{\Delta}_{k, \ell}^+ }^2
    }
    \\
    &\leq
    \exp
    \prn*{
    -2
    \sum_{a_k \in \mathcal{A}^{\ell}}
    n_k
    \prn*{ \hat{\Delta}_{k, \ell}^{+} }^2
    }
    .
\end{align*}
Where we drop $(1 - \alpha)$, and take sum over $\mathcal{A}^{\ell}$ instead since $\epsilon_{k, \ell} \leq 1$. Arms in $\mathcal{A}^{\ell} \subset \mathcal{A}$ (Definition~\ref{def:A_ell_definition}) have the feature of having $(1 - \alpha) \mu_k > \mu_{\ell}$. Next\footnote{We start denoting $\beta_k (t, \delta) = \sqrt{\frac{1}{2n} \log \prn*{\frac{1}{\delta}}}$ by just $\beta_k$ for any arbitrary arm $a_k \in \mathcal{A}$ unless the parameters are needed}, to upper bound $\prod \epsilon_{k, \ell}$, we lower bound each $\hat{\Delta}_{k, \ell}^+$ by leveraging clauses B and C of $G_{i, \ell}$.
\begin{align*}
    (1 - \alpha) \hat{\mu}_k 
    &\geq
    (1 - \alpha)
    \prn*{\mu_k - \beta_k}
    \,\,
    \forall \,
    a_k \in \mathcal{A}^{\ell}
    ,
    \\
    \UCB_{\ell}
    &\leq
    \mu_{\ell} + 2 \beta_{\ell}
    \\
    \implies
    \hat{\Delta}_{k, \ell}
    &=
    (1 - \alpha) \hat{\mu}_k - \UCB_{\ell}
    \\
    &\geq (1 - \alpha) \mu_k - (1 - \alpha) \beta_k
    -\mu_{\ell} - 2 \beta_{\ell}
    \\
    &\geq
    \Delta_{k, \ell} - \beta_k - 2 \beta_{\ell}
    \quad
    \text{(using Definition~\ref{def:A_ell_definition} for $\Delta_{k, \ell}$)}
    \\
    \implies
    \prn*{\hat{\Delta}_{k, \ell}}^{+}
    &\geq
    \max \crl*{ \Delta_{k, \ell} - \beta_k - 2 \beta_{\ell}, 0 }
    =
    \prn*{
    \Delta_{k, \ell} - \beta_k - 2 \beta_{\ell}
    }^{+}
    .
\end{align*}
Therefore, we can continue to bound $\prod \epsilon_{k, \ell}$ as,
\begin{align*}
    \prod_{a_k \in \mathcal{A}}
    \epsilon_{k, \ell}
    &\leq
    \exp 
    \prn*{
    - 2 
    \sum_{a_k \in \mathcal{A}^{\ell}}
    n_k 
    \prn*{
    \prn*{
    \Delta_{k, \ell}
    -
    \beta_k
    -
    2 \beta_{\ell}
    }^+
    }^2
    }
    .
\end{align*}
For any arm $a_k \in \mathcal{A}^{\ell}$, conditioned on the event $G_{i, \ell}$, we have,
\begin{align}
\UCB_k 
= \hat{\mu}_k + \beta_k
\leq \mu_k + 2 \beta_k
\,\,
\forall 
\,
t \in [T]
\label{eq:ub_for_ucb_k}
.
\end{align}
Combining \ref{eq:clause_C_part2} and \ref{eq:ub_for_ucb_k} we can see that due to the BAI-filter, $a_k$ can have at most $\frac{8 \log \prn*{1 / \delta} }{ \Delta_k^2 }$ samples (analogous to the derivation of $\tau^{\BAI}_{i, \ell}$). Let $n$ represent the running count of the sampling rounds completed during episode $\ell$. In each sampling round the unfiltered arms of $\mathcal{G}_{\ell}$ are sampled. Then for $a_{\ell}$ we will have $n_{\ell} \geq n$. Similar to $a_\ell$, for each $a_k \in \mathcal{A}^{\ell}$ that was not filtered $n_k \geq n$. Further, let $\beta = \sqrt{\frac{1}{2n} \log \prn*{\frac{1}{\delta}}}$ be the confidence radius corresponding to $n$ and error tolerance $\delta$. Then $\beta \geq \max \crl*{\beta_{\ell}, \beta_k}$, and,
\begin{align*}
\prn*{ \Delta_{k, \ell} - \beta_k - 2 \beta_{\ell} }^+
= 
\max \prn*{ \Delta_{k, \ell} - \beta_k - 2 \beta_{\ell}, 0 }
\geq \max \prn*{ \Delta_{k, \ell} - 3 \beta, 0 } 
=
\prn*{
\Delta_{k, \ell} - 3 \beta
}^{+}
.
\end{align*}
We can use this fact to further upper bound $\prod \epsilon_{k, \ell}$.
\begin{align*}
\prod_{a_k \in \mathcal{A}} \epsilon_{k, \ell}
&\leq
\exp 
\prn*{
-2 n
\sum_{a_k \in \mathcal{A}^{\ell}}
\Ind\crl*{n \leq \frac{8 \log \prn*{1 / \delta} }{ \Delta_k^2 } }
\prn*{
\prn*{
\Delta_{k, \ell}
- 3 \beta
}^+
}^2
}
.
\end{align*}
To get our desired bound we wish to find the smallest integer $\tau_{\ell}$ such that (conditioned on $G_{i, \ell}$), $a_\ell$ is deemed infeasible $\forall \,\, n \geq \tau_{\ell}$.
\begin{align}
\tau_{\ell} (\delta)
&=
\min 
\crl*{ 
n 
\scalebox{2.0}{$\mid$}
\exp 
\prn*{
-2 n
\sum_{a_k \in \mathcal{A}^{\ell}}
\Ind\crl*{n \leq \frac{8 \log \prn*{1 / \delta} }{ \Delta_k^2 } }
\prn*{
\prn*{
\Delta_{k, \ell}
- 3 \beta
}^+
}^2
}
\leq 
\delta
}
\nonumber
\\
&=
\min
\crl*{
n
\scalebox{2.0}{$\mid$}
\sum_{a_k \in \mathcal{A}^{\ell}}
\Ind\crl*{n \leq \frac{8 \log \prn*{1 / \delta} }{ \Delta_k^2 } }
\prn*{
\prn*{
\Delta_{k, \ell}
- 3 \beta
}^+
}^2
\geq 
\beta^2
}
\quad
\text{(using definition of $\beta$ and rearranging)}
\label{eq:core_optimization_problem_tau_ell}
.
\end{align}

We define $\mathcal{A}^{\ell} (p)$ as the collection of the top $p$ reward arms in $\mathcal{A}^{\ell}$. Reintroducing the reward ordered indexing scheme $\phi(i)$ first introduced in the statement of Lemma~\ref{lemma:expensive_in_episode_ell}, 
\begin{align*}
    \mathcal{A}^{\ell} (p)
    &\coloneqq
    \crl*{
    a_{\phi(1)},
    a_{\phi(2)},
    \ldots,
    a_{\phi(p)}
    }
    .
\end{align*}
Where as in the statement of Lemma~\ref{lemma:expensive_in_episode_ell}, $a_{\phi(i)}$ is the $i^{\text{th}}$ highest reward arm in $\mathcal{A}^{\ell}$. To solve the optimization problem of \ref{eq:core_optimization_problem_tau_ell}, we define a candidate solution $\tau_{\ell, p}$ for $\tau_{\ell}$ as,
\begin{align*}
    \tau_{\ell, p}
    &\coloneqq
    \min
    \crl*{
    n
    \scalebox{2.0}{$\mid$}
    \sum_{a_k \in \mathcal{A}^{\ell}(p) }
    \prn*{ \Delta_{k, \ell} - 3 \beta }^2
    \geq
    \beta^2
    }
    .
\end{align*}
We say that candidate solution $\tau_{\ell, p}$ is feasible when,
\begin{align*}
    \sum_{a_k \in \mathcal{A}^{\ell} (p)}
    \prn*{ \Delta_{k, \ell} - 3 \beta \prn*{\tau_{\ell, p}, \delta} }^2
    =
    \sum_{a_k \in \mathcal{A}^{\ell}}
    \Ind \crl*{ \tau_{\ell, p} \leq \frac{8 \log \prn*{1 / \delta} }{ \Delta_k^2 } }
    \cdot
    \prn*{
    \prn*{ \Delta_{k, \ell} - 3 \beta \prn*{\tau_{\ell, p}, \delta} }^{+}
    }^{2}
    .
\end{align*}
We can get the desired $\tau_{\ell}$ as the minimum over the feasible solutions,
\begin{align*}
    \tau_{\ell}
    (\delta)
    &=
    \min
    \crl*{ \tau_{\ell, p} (\delta)
    \mid
    p = 1, \ldots, \abs*{\mathcal{A}^{\ell}}
    \text{, and }
    \tau_{\ell, p} 
    \text{ is feasible}
    }
    .
\end{align*}
Using Lemma~\ref{lemma:optimization_results_for_infeasible_end} we can then say that each candidate $\tau_{\ell, p}$ is given by,
\begin{align*}
    \tau_{\ell, p}
    (\delta)
    &=
    \frac{ \prn*{3 \sqrt{p} + 1}^2 }{ 2 }
    \frac{\log \prn*{1 / \delta} }{ \sum_{i=1}^{p}  \Delta_{\phi(i), \ell}^2 }
    .
\end{align*}
In general the best feasible $\tau_{\ell, p}$ will depend on the bandit instance. If we define $A$ such that $\tau_{\ell, A}$ is the smallest feasible solution, then,
\begin{align*}
    \tau_{\ell} (\delta)
    &=
    \frac{ \prn*{3 \sqrt{A} + 1}^2 }{2}
    \frac{\log \prn*{1 / \delta} }{ \sum_{i=1}^{A} \Delta_{\phi(i), \ell}^2 }
    .
\end{align*}
Conditioned on $G_{i, \ell}$, overall the expected number of samples are bound as,
\begin{align*}
    \Exp_{\ell} 
    \brk*{
    n_i (T)
    \mid
    G_{i, \ell}
    }
    &\leq
    \min 
    \crl*{
    \tau^{\textrm{BAI}}_{i},
    \tau_{\ell}
    }
    \\
    &\leq
    \min 
    \crl*{
    \frac{8\log\prn*{ 1 / \delta }}
    {\Delta_i^2}, 
    \frac{ \prn*{3 \sqrt{A} + 1}^2 }{2}
    \frac{\log \prn*{1 / \delta} }{ \sum_{i=1}^{A} \Delta_{\phi(i), \ell}^2 }
    }
    .
\end{align*}
To complete the bound on $\Exp_{\ell} \brk*{ n_i (T) }$ from Equation~\ref{eq:initial_sample_bound_expensive_i} we must also bound $\Pr{ G_{i, \ell}^c }$. We start by combining the definition of $G_{i, \ell}$ with De-morgan's rule.
\begin{align*}
    \Pr{
    G_{i, \ell}^c
    }
    &=
    \Pr{
    \bigcup_{t = 1}^{T}
    \crl*{ 
    \abs*{ \hat{\mu}_i - \mu_i } 
    \geq \beta_i (t, \delta) 
    }
    \cup
    \bigcup_{t = 1}^{T}
    \crl*{ 
    \abs*{ \hat{\mu}_{\ell} - \mu_{\ell} }
    \geq \beta_{\ell} (t, \delta)
    }
    \cup
    \bigcup_{t = 1}^{T}
    \crl*{
    \bigcap_{a_k \in \mathcal{A}^{\ell}}
    \crl*{
        \abs{ \hat{\mu}_k  - \mu_k }
        \geq
        \beta_k ( t, \delta )
    }    
    }
    }
    \\
    &\leq
    \sum_{t = 1}^{T}
    \Pr{
    \abs*{ \hat{\mu}_i - \mu_i } 
    \geq \beta_i (t, \delta)
    }
    +
    \sum_{t=1}^{T}
    \Pr{ 
    \abs*{ \hat{\mu}_{\ell} - \mu_{\ell} }
    \geq \beta_{\ell} (t, \delta)
    }
    +
    \sum_{t=1}^{T}
    \sum_{a_k \in \mathcal{A}^{\ell}}
    \Pr{
        \abs{ \hat{\mu}_k  - \mu_k }
        \geq
        \beta_k ( t, \delta ) 
    }
    \\
    &\leq
    K T \delta
    .
\end{align*}
Plugging back in this bound on $\Pr { G_{i, \ell}^c }$ and the earlier bound on $\Exp_{\ell} \brk*{n_i (T) \mid G_{i, \ell}}$ into Equation~\ref{eq:initial_sample_bound_expensive_i} we get the bound stated in Lemma~\ref{lemma:expensive_in_episode_ell}.
\end{proof}

\begin{remark}[Role of gating filter]
Line 3 of Algorithm~\ref{algo:COF} resets the gating arms at the start of every iteration of COF and in effect acts like a gating filter. This remark addresses why the gating filter does not play a role in the upper bound on the expected number of samples of an expensive arm during episodes $\ell < a^*$.

For episodes $\ell < a^*$, $\mu_{\ell} < (1 - \alpha) \mu^*$. If for some arm $a_i$, $(1 - \alpha) \mu_i > \mu_{\ell}$ then it is a legitimate member of the gating set $\mathcal{G}_{\ell}$ and not at risk from being removed from sampling by the gating filter. In the other case $\mu_{\ell} > (1 - \alpha) \mu_i$ and we can study the gap $\mu_{\ell} - (1 - \alpha) \mu_i$ to understand the filtering impact.
\begin{align*}
\mu_{\ell} 
- 
(1 - \alpha) \mu_i 
&<
(1 - \alpha) \mu^*
-
\mu_i
\\
&<
\mu^*
-
\mu_i
=
\Delta_i
\end{align*}
Therefore from the perspective of sample complexity analysis of an arbitrary expensive arm $a_i$, BAI-filter will always kick in before gating set filter. 

However in the case of expected samples of $a_i$ during episode $a^*$ gating filter does play a role as shown in Lemma~\ref{lemma:expensive_episode_a_star_bound}.
\end{remark}

\begin{lemma}[Bound on number of samples of an expensive arm in episode $a^*$]
Under COF the expected number of samples of expensive arm $a_i \in \mathcal{A}^{+}$, conditioned on its final sample being drawn in episode $a^*$, and on the event $G_{a^*}$ (defined in Equation~\ref{eq:G_a_star_definition}), over horizon $T$ is upper bounded as,
\begin{align*}
    \Exp
    \brk*{
    n_i (T)
    \mid
    E_{i, a^*},
    G_{a^*}
    }
    &\leq
    \frac{8 \log \prn*{1 / \delta} }
    { \prn*{ \mu_{a^*} - (1 - \alpha) \mu_i }^2 }
    + T^2 \delta
    .
\end{align*}
\label{lemma:expensive_episode_a_star_bound}
\end{lemma}
\begin{proof}
Corresponding to episode $a^*$ for the samples $n_i (T)$ we wish to bound, $\Exp \brk*{ n_i (T) \mid E_{i, a^*}, G_{a^*}}$. In addition to the normativity secured by conditioning on $G_{a^*}$, we must also condition on an additional event to secure the normativity around arm $i$'s elimination by arm $a^*$. Hence we define $G_{i, a^*}$ as,
\begin{align*}
    G_{i, a^*}
    &=
    \crl*{ 
    \abs{ \hat{\mu}_{i} - \mu_{i} } 
    < 
    \beta_{i} (t, \delta) 
    \,\, \forall \, t \in [T] 
    }
    \\
    &=
    \bigcap_{t = 1}^T
    \crl*{ 
    \abs{ \hat{\mu}_{i} - \mu_{i} } 
    < 
    \beta_{i} (t, \delta) 
    }
    .
\end{align*}
We bound $\Exp \brk*{ n_i (T) \mid E_{i, a^*}, G_{a^*}}$ by further conditioning on $G_{i, a^*}$ using Lemma~\ref{lemma:bound_by_conditioning},
\begin{align}
    \Exp_{a^*}
    \brk*{ n_i (T) \mid G_{a^*} }
    &\leq
    \Exp_{a^*}
    \brk{ n_i (T) \mid G_{a^*}, G_{i, a^*} }
    +
    T 
    \cdot
    \Pr{ G_{i, a^*}^c \mid E_{i, a^*}, G_{a^*} }
    \nonumber
    \\
    &=
    \Exp_{a^*}
    \brk{ n_i (T) \mid G_{a^*}, G_{i, a^*} }
    +
    T 
    \cdot
    \Pr{ G_{i, a^*}^c }
    \quad 
    \text{ (arms are independent, $G_{i, a^*}$ is agnostic of episode)}
    \nonumber
    \\
    &\leq
    \Exp_{a^*}
    \brk{ n_i (T) \mid G_{a^*}, G_{i, a^*} }
    +
    T^2 \delta
    \quad 
    \text{ (using union bound and Lemma~\ref{lemma:ucb_lcb_incorrect})}
    .
    \label{eq:initial_bound_for_ep_a_star}
\end{align}
Where we have introduced the $\Exp_{a^*} \brk*{ \cdot } \coloneqq \Exp \brk*{ \cdot \mid E_{i, a^*} }$ operator for brevity. Next we show that conditioned on $G_{i, a^*}, G_{a^*}$, the number of samples of $a_i$ during episode $a^*$ can be at most a certain bandit instance dependent quantity $\tau_{i, a^*}$.

When $G_{i, a^*}$ and $G_{a^*}$ hold, we can write,
\begin{align}
    \LCB_{a^*}
    &=
    \hat{\mu}_{a^*} 
    -
    \beta_{a^*} (t, \delta)
    \geq
    \mu_{a^*}
    -
    2 \beta_{a^*} (t, \delta)
    \label{eq:a_star_lcb_ub}
    \\
    \UCB_i
    &=
    \hat{\mu}_i + \beta_i (t, \delta)
    \leq
    \mu_i 
    +
    2 \beta_i (t, \delta)
    \label{eq:arm_i_ucb_ub}
    .
\end{align}
Since the lower bound on $\LCB_{a^*}$ defined in Equation \ref{eq:a_star_lcb_ub} and the upper bound on $\UCB_{i}$ defined in equation \ref{eq:arm_i_ucb_ub} are strictly increasing and strictly decreasing functions of time $t$ (or equivalently number of samples $n_{a^*} (t)$ and $n_i (t)$ respectively), once a strict separation is achieved between the lower bound from Equation \ref{eq:a_star_lcb_ub} and the upper bound from Equation \ref{eq:arm_i_ucb_ub}, arm $a_i$ will be eliminated from the set of gating arms for episode $a^*$ and will not be sampled again during that episode.

Hence conditioned on $G_{a^*}, G_{i, a^*}$ we can find the maximum possible samples by plugging in the filtering condition on $\mathcal{G}_{a^*}$ implemented by COF (Line 3, Algorithm~\ref{algo:COF}). 
\begin{align*}
    (1 - \alpha) 
    \UCB_i
    \leq
    (1 - \alpha)
    \prn*{
        \mu_i
        +
        2 \beta_i (t, \delta)
    }
    &<
    \mu_{a^*}
    -
    2 \beta_{a^*} (t, \delta)
    \leq
    \LCB_{a^*}
    .
\end{align*}
Rearranging terms we have,
\begin{align*}
    2 (1 - \alpha)
    \beta_i (t, \delta)
    +
    2 \beta_{a^*} (t, \delta)
    &\leq
    \mu_{a^*}
    -
    (1 - \alpha) \mu_i
    .
\end{align*}
Since during episode $a^*$, $n_{a^*} (t) \geq n_i (t)$, $\beta_i (t, \delta) > \beta_{a^*} (t, \delta)$, the inequality will be satisfied when,
\begin{align*}
    2(2 - \alpha) 
    \beta_i (t, \delta)
    &<
    \mu_{a^*}
    -
    (1 - \alpha)
    \mu_i
    .
\end{align*}
Hence the required value of $\beta_i (t, \delta)$ is governed by,
\begin{align*}
    \beta_{i} (t, \delta)
    &<
    \frac{\mu_{a^*} - (1 - \alpha) \mu_i}
    {2(2 - \alpha) }
    .
\end{align*}
Hence, the maximum number of samples of $n_i$ is governed by,
\begin{align*}
    \sqrt{
    \frac{1}{2 n_{i}}
    \log
    \prn*{
        \frac{1}{\delta}
    }
    }
    <
    \frac{\mu_{a^*}
    -
    (1 - \alpha) \mu_i}{2(2 - \alpha) }
    .
\end{align*}
We can choose the least value that will satisfy the above condition by setting equality,
\begin{align*}
    \tau_{i, a^*}
    &=
    \frac{2 (2 - \alpha)^2 \log \prn*{1 / \delta}}
    {
    \prn*{
    \mu_{a^*}
    -
    (1 - \alpha) \mu_i
    }^2
    }
    \\
    &\leq
    \frac{8 \log \prn*{1 / \delta}}
    {
    \prn*{
    \mu_{a^*}
    -
    (1 - \alpha)
    \mu_i}^2
    }
    .
\end{align*}
Hence we can further bound the RHS of Equation~\ref{eq:initial_bound_for_ep_a_star} as,
\begin{align}
    \Exp_{a^*}
    \brk*{ n_i (T) \mid G_{a^*} }
    &\leq
    \frac{8 \log \prn*{1 / \delta}}
    {
    \prn*{
    \mu_{a^*} - (1 - \alpha) \mu_i}^2
    }
    +
    T^2 \delta
    .
    \label{eq:bound_for_ep_a_star}
\end{align}
\end{proof}

We are now close to ready to bound the overall number of samples of an expensive arm. Before we can, we need a final intermediate result to obtain an ordering between the number of sampling rounds required to disqualify the best reward cheap arm $a_{\dagger}$ and some other arbitrary cheap arm $a_\ell, \,\, \ell \neq j$.
\begin{lemma}[Disqualifying $a_{\dagger}$ requires the most samples]
Let $a_{\dagger}$ denote the highest reward cheap arm. In the notation of Lemma~\ref{lemma:expensive_in_episode_ell} let $\tau_\ell (\delta)$ denote the least number of samples from each participating arm (sampling rounds) needed to deem $a_\ell$ infeasible with an error tolerance $\delta$. Then,
\begin{align*}
    \tau_{\dagger} (\delta) 
    = 
    \max_{\ell < a^*} 
    \tau_{\ell} (\delta)
    .
\end{align*}
\label{lemma:episode_ordering_lemma}
\end{lemma}
\begin{proof}
 By definition $\tau_{\dagger}$ satisfies Equation~\ref{eq:core_optimization_problem_tau_ell},    
\begin{align*}
\tau_{\dagger}
(\delta)
&=
\min
\crl*{
n
\scalebox{2.0}{$\mid$}
\sum_{a_k \in \mathcal{A}^{\dagger}}
\Ind\crl*{n \leq \frac{8 \log \prn*{1 / \delta} }{ \Delta_k^2 } }
\prn*{
\prn*{
\Delta_{k, j}
- 3 \beta
}^+
}^2
\geq 
\beta^2
}
.
\end{align*}
To prove Lemma~\ref{lemma:episode_ordering_lemma}, we rely on two observations,
\begin{enumerate}
    \item As remarked in \ref{remark:the_set_A_dagger}, $\mathcal{A}^{\dagger} \subseteq \mathcal{A}^{\ell}$. That is, any arm included in $\mathcal{A}^{\dagger}$ is also a part of other $\mathcal{A}^{\ell}$ for any cheap $a_{\ell}$.
    \item Since $\prn*{\Delta_{k, \ell} - 3\beta}^{+} = \max \crl*{\Delta_{k, \ell} - 3 \beta, 0}$, and $\Delta_{k, \ell} = (1 - \alpha) \mu_k - \mu_{\ell}$, we have $\Delta_{k, \dagger} \leq \Delta_{k, \ell} \,\, \forall \,\,  \ell < a^*$. 
\end{enumerate}
Starting with the general feasibility condition for $\tau_{\ell}$ and applying our two observations,
\begin{align*}
\sum_{a_k \in \mathcal{A}^{\ell}}
\Ind\crl*{n \leq \frac{8 \log \prn*{1 / \delta} }{ \Delta_k^2 } }
\prn*{
\prn*{
\Delta_{k, \ell}
- 3 \beta
}^+
}^2
&\geq
\sum_{a_k \in \mathcal{A}^{\dagger}}
\Ind\crl*{n \leq \frac{8 \log \prn*{1 / \delta} }{ \Delta_k^2 } }
\prn*{
\prn*{
\Delta_{k, \dagger}
- 3 \beta
}^+
}^2
.
\end{align*}
Which in turn means that $\tau_{\ell} (\delta)$ as defined in Equation~\ref{eq:core_optimization_problem_tau_ell} is not more than $\tau_{\dagger} (\delta)$.
\end{proof}

\begin{lemma}[Bound on number of samples of an expensive arm]
    When the error tolerance $\delta = T^{-2}$, the expected number of samples of expensive arm $a_i \in \mathcal{A}^{+}$ under COF over horizon $T$ is upper bounded as,
    \begin{align*}
    \Exp
    \brk*{
    n_i (T)
    }
    &\leq
    \max
    \crl*{
    \min
    \crl*{
    \tau_{\dagger} (T^{-2}),
    \frac{16 \log T }{\Delta_i^2}
    },
    \,\,
    \frac{16 \log T }
    {
    \prn*{\mu_{a^*} - (1 - \alpha) \mu_i}^2
    }
    }
    + K
    .
    \end{align*}
    Where $a_{\dagger} = \arg \max_{a_k \in \mathcal{A}^{-}} \mu_k$ is the best reward cheap arm, and $\tau_{\dagger}$ is as defined in Lemma~\ref{lemma:expensive_in_episode_ell}.
    \label{lemma:expensive_overall_ub}
\end{lemma}
\begin{proof}
Plugging in the bounds from Lemmas~\ref{lemma:expensive_in_episode_ell} (for episodes $\ell < a^*$) and \ref{lemma:expensive_episode_a_star_bound} (for episode $a^*$) into the max over episodes bound from Lemma~\ref{lemma:normative_episode_a_star} we get,
\begin{align*}
    \Exp \brk*{ n_i (T) }
    &\leq
    \max
    \crl*{
    \max_{\ell < a^*}
    \crl*{
    \min
    \crl*{
    \frac{8 \log \prn*{1 / \delta} }{\Delta_i^2},
    \tau_{\ell} (\delta)
    }
    + K T^2 \delta
    },
    \,\,
    \frac{8 \log \prn*{1 / \delta}}
    {
    \prn*{\mu_{a^*} - (1 - \alpha) \mu_i}^2
    }
    +
    3 T^2 \delta
    }
    \\
    &\leq
    \max
    \crl*{
    \max_{\ell < a^*}
    \crl*{
    \min 
    \crl*{
    \frac{8 \log \prn*{1 / \delta} }{\Delta_i^2},
    \tau_{\ell} (\delta)
    }
    },
    \,\,
    \frac{8 \log \prn*{1 / \delta} }
    {
    \prn*{\mu_{a^*} - (1 - \alpha) \mu_i}^2
    }
    }
    +
    K T^2 \delta
    \quad
    \text{(pulling out larger error term)}
    \\
    &\leq
    \max
    \crl*{
    \min
    \crl*{
    \max_{\ell < a^*} 
    \crl*{\tau_{\ell} (\delta)},
    \frac{8 \log \prn*{1 / \delta} }{\Delta_i^2}
    },
    \,\,
    \frac{8 \log \prn*{1 / \delta}}
    {
    \prn*{\mu_{a^*} - (1 - \alpha) \mu_i}^2
    }
    }
    +
    K T^2 \delta
    \quad
    \text{(using Lemma~\ref{lemma:max_min_lemma})}
    .
    \\
    &=
    \max
    \crl*{
    \min
    \crl*{
    \tau_{\dagger} (\delta),
    \frac{8 \log \prn*{1 / \delta} }{\Delta_i^2}
    },
    \,\,
    \frac{8 \log \prn*{1 / \delta}}
    {
    \prn*{\mu_{a^*} - (1 - \alpha) \mu_i}^2
    }
    }
    +
    K T^2 \delta
    \quad
    \text{(using Lemma~\ref{lemma:episode_ordering_lemma})}
    .
\end{align*}
Plugging in the error tolerance $\delta = T^{-2}$, we obtain the expression stated in Lemma~\ref{lemma:expensive_overall_ub}.
\end{proof}

\subsection{Bound Cost and Quality Regret}
\label{eq:cof_regret_bounds}

We put together the bound on the samples of cheap and expensive arms together to prove the upper bound on cost and quality regret stated in Theorem~\ref{thm:cof_regret_ub}.
\begin{proof}[Proof of Theorem~\ref{thm:cof_regret_ub}]
Using the regret decomposition (Lemma~\ref{lemma:regret_decomposition_lemma}) we can express and bound the expected cumulative cost regret, and the expected cumulative quality regret as,
\begin{align*}
\Exp
\brk*{
\costRegret
\prn*{ T, \nu }
}
&=
\sum_{a_i \in \mathcal{A}^+}
\Delta_{C, i}^{+}
\Exp 
\brk*{
n_i \prn*{T}
}
,
\quad
\text{ ($\Delta_{C, i}^+ = 0$ for cheap arms)}
\\
\Exp
\brk*{
\qualityRegret
\prn*{T, \nu}
}
&=
\sum_{a_i \in \mathcal{A}^-}
\Delta_{Q, i}^+ 
\Exp \brk*{ n_i (T) }
+
\sum_{a_i \in \mathcal{A}^+}
\Delta_{Q, i}^+ 
\Exp \brk*{ n_i (T) }
,
\end{align*}
respectively. To bound $\Exp \brk*{n_i (T)}$ for cheap arms ($a_i \in \mathcal{A}^-$) we invoke Lemma~\ref{lemma:cheap_arms_ub}, and for expensive arms $(a_i \in \mathcal{A}^+)$ we use Lemma~\ref{lemma:expensive_overall_ub}.
\begin{align*}
    \Exp
    \brk*{
    \costRegret
    \prn*{T, \nu}
    }
    &\leq
    \sum_{a_i \in \mathcal{A}^+}
    \Delta_{C, i}^+
    \prn*{
    \max
    \crl*{
    \min
    \crl*{
    \tau_{\dagger} (T^{-2}),
    \frac{16 \log T }{\Delta_i^2}
    },
    \,\,
    \frac{16 \log T }
    {
    \prn*{\mu_{a^*} - (1 - \alpha) \mu_i}^2
    }
    }
    +
    K
    }
    \\
    &=
    \sum_{a_i \in \mathcal{A}^+}
    \Delta_{C, i}^+ 
    \prn*{
    \max
    \crl*{
    \min
    \crl*{
    \tau_{\dagger} (T^{-2}),
    \frac{16 \log T }{\Delta_i^2}
    },
    \,\,
    \frac{16 \log T }
    {
    \prn*{\mu_{a^*} - (1 - \alpha) \mu_i}^2
    }
    }
    }
    +
    K \sum_{a_i \in \mathcal{A}^+}
    \Delta_{C, i}^+
    .
\end{align*}
Similarly, for quality regret,
\begin{align*}
    \Exp \brk*{ \qualityRegret \prn*{T, \nu} }
    &\leq
    \sum_{a_i \in \mathcal{A}^-}
    \Delta_{Q, i}^+ \prn*{ \frac{16 \log T}{\Delta_{Q, i}^2} + 2} 
    \\
    &\quad+
    \sum_{ a_i \in \mathcal{A}^+ }
    \Delta_{Q, i}^+ 
    \prn*{
    \max
    \crl*{
    \min
    \crl*{
    \tau_{\dagger} (T^{-2}),
    \frac{16 \log T }{\Delta_i^2}
    },
    \,\,
    \frac{16 \log T }
    {
    \prn*{\mu_{a^*} - (1 - \alpha) \mu_i}^2
    }
    }
    +
    K
    }
    \\
    &=
    \sum_{a_i \in \mathcal{A}^-}
    \frac{16 \log T}{\Delta_{Q, i}^+}
    +
    \sum_{a_i \in \mathcal{A}^+}
    \Delta_{Q, i}^+
    \prn*{
    \max
    \crl*{
    \min
    \crl*{
    \tau_{\dagger} (T^{-2}),
    \frac{16 \log T }{\Delta_i^2}
    },
    \,\,
    \frac{16 \log T }
    {
    \prn*{\mu_{a^*} - (1 - \alpha) \mu_i}^2
    }
    }
    }
    \\
    &\quad+ 
    2 \sum_{a_i \in \mathcal{A}^-} \Delta_{Q, i}^+
    +
    K \sum_{a_i \in \mathcal{A}^+} \Delta_{Q, i}^+
    \\
    &\leq
    \sum_{a_i \in \mathcal{A}^-}
    \frac{16 \log T}{\Delta_{Q, i}^+}
    +
    \sum_{a_i \in \mathcal{A}^+}
    \Delta_{Q, i}^+
    \prn*{
    \max
    \crl*{
    \min
    \crl*{
    \tau_{\dagger} (T^{-2}),
    \frac{16 \log T }{\Delta_i^2}
    },
    \,\,
    \frac{16 \log T }
    {
    \prn*{\mu_{a^*} - (1 - \alpha) \mu_i}^2
    }
    }
    }
    \\
    &\quad+
    K \sum_{a_i \in \mathcal{A}} \Delta_{Q, i}^+
    .
\end{align*}
To obtain the bound stated in Theorem~\ref{thm:cof_regret_ub} we substitute $\displaystyle \tau_{\dagger} (T^{-2}) = \frac{(3\sqrt{A} + 1)^2 \log T}{\sum_{i=1}^{A} \Delta_{\phi(i), \dagger}^2}$.
\end{proof}


\end{document}